\documentclass[reqno]{amsart}
\usepackage[margin=1.25in,bottom=1.25in]{geometry}		% for tighter margins (80 CPL)

\PassOptionsToPackage{numbers, sort&compress}{natbib}
\PassOptionsToPackage{svgnames, dvipsnames}{xcolor}
\PassOptionsToPackage{compatV3}{fancyhdr}
\usepackage{xparse}

\usepackage{amsmath}		% for AMS macros
\usepackage{amssymb}		% for AMS symbols
\usepackage{amsfonts}		% for AMS fonts
\usepackage{amsthm}		% for theorems
\usepackage{mathtools}		% for advanced math
\usepackage[foot]{amsaddr}		% for author footnotes
\usepackage{graphicx}
\mathtoolsset{%
}

\usepackage[utf8]{inputenc}		% for source encoding
\usepackage[T1]{fontenc}		% for font encoding

\usepackage[%		% for math font selection
cal=cm,
]
{mathalfa}

\usepackage{dsfont}		% for blackboard bold font
\usepackage{acronym}		% for acronyms
\renewcommand*{\aclabelfont}[1]{\acsfont{#1}}		% for acronym label font
\usepackage{contour}
\usepackage[labelfont={bf,small},labelsep=colon,font=small]{caption}	% for caption control
\usepackage{subcaption}		% for subfigures
\usepackage{wrapfig}

\usepackage[dvipsnames,svgnames]{xcolor}		% for color
\colorlet{MyRed}{FireBrick!50!Crimson}
\colorlet{MyBlue}{DodgerBlue!75!black}
\colorlet{MyGreen}{DarkGreen!85!black}
\colorlet{MyViolet}{DarkMagenta}

\colorlet{MyLightBlue}{DodgerBlue!20}
\colorlet{MyLightGreen}{MyGreen!20}

\colorlet{PrimalColor}{MyBlue}
\colorlet{PrimalFill}{MyLightBlue}
\colorlet{DualColor}{MyRed}

\colorlet{RevColor}{MyRed}
\colorlet{LinkColor}{MediumBlue}

\newcommand{\ackperiod}{}		% for AMS bug in acknowledgments

\usepackage{cancel}		% for cancelling terms
\usepackage{latexsym}		% for symbols

\usepackage{pifont}		% for dingbats
\usepackage{float}		% for placing figures beside their lemma statements
\usepackage{needspace}		% reserve room for a statement and its illustration
\usepackage{tikz}		% for figures
\usepackage{tikz-cd}		% for figures
\usetikzlibrary{calc,patterns}		% for basic tikz figures
\usepackage{pgfplots}
\pgfplotsset{compat=1.18}

\usepackage{array}		% for flexible tables
\usepackage{booktabs}		% for better tables
\usepackage[inline,shortlabels]{enumitem}		% for lists
\setenumerate{itemsep=\smallskipamount,topsep=\smallskipamount,left=\parindent}
\setitemize{itemsep=\smallskipamount,topsep=\smallskipamount,left=\parindent}
\usepackage[kerning=true]{microtype}		% for microtypography

\usepackage{tabto}		% for spacing
\usepackage{xspace}		% for flexible spaces

\usepackage{natbib}		% for citations
\usepackage{hyperref}
\hypersetup{
colorlinks=true,
linktocpage=true,
pdfstartview=FitH,
breaklinks=true,
pdfpagemode=UseNone,
pageanchor=true,
pdfpagemode=UseOutlines,
plainpages=false,
bookmarksnumbered,
bookmarksopen=false,
bookmarksopenlevel=1,
hypertexnames=false,
pdfhighlight=/O,
urlcolor=LinkColor,linkcolor=LinkColor,citecolor=LinkColor,	% for on-screen
pdftitle={Chaos persists in large-scale multi-agent learning despite adaptive learning rates},
pdfauthor={Emmanouil V. Vlatakis-Gkaragkounis; Konstantinos Metaxas; Lampros Flokas; Georgios Piliouras},
pdfsubject={},
pdfkeywords={},
pdfcreator={pdfLaTeX},
pdfproducer={LaTeX with hyperref}
}

\usepackage[sort&compress,capitalize,nameinlink]{cleveref}		% for cleveref formatting
\crefname{algo}{Algorithm}{Algorithms}
\crefname{assumption}{Assumption}{Assumptions}
\crefname{case}{Case}{Cases}
\usepackage{algpseudocode}		% for algorithm macros
\renewcommand{\algorithmiccomment}[1]{\hfill\textsf{\#} #1}		% for algorithm comments

\usepackage{thmtools}		% for theorem tools
\usepackage{thm-restate}		% for restating theorems

\theoremstyle{plain}
\newtheorem{theorem}{Theorem}		% for theorems
\newtheorem{corollary}{Corollary}		% for corollaries
\newtheorem{lemma}{Lemma}		% for lemmas
\newtheorem{claim}{Claim}		% for claims
\newtheorem{fact}{Fact}		% for claims

\newtheorem*{corollary*}{Corollary}		% for corollaries (unnumbered)

\theoremstyle{definition}
\newtheorem{definition}{Definition}		% for definitions
\newtheorem*{definition*}{Definition}		% for definitions (unnumbered)
\newtheorem*{assumption*}{Assumptions}		% for assumptions (unnumbered)
\newtheorem*{example*}{Example}		% for examples (unnumbered)

\theoremstyle{remark}
\newtheorem*{remark*}{Remark}		% for remarks (unnumbered)

\newcounter{proofpart}

\numberwithin{example}{section}		% for example numbering

\usepackage[suppress]{color-edits}		% final version: hide editing markup
\usepackage[normalem]{ulem}		% for strikeout text

\newcommand{\debug}[1]{#1}		% final version: no macro coloring

\newcommand{\newmacro}[2]{\newcommand{#1}{\debug{#2}}}		% for shorthand definitions
\newcommand{\newop}[2]{\DeclareMathOperator{#1}{\debug{#2}}}		% for shorthand definitions

\DeclarePairedDelimiter{\abs}{\lvert}{\rvert}		% for absolute value
\DeclarePairedDelimiterX{\setdef}[2]{\{}{\}}{#1:#2}		% for set builder notation
\DeclarePairedDelimiterXPP{\exclude}[1]{\mathopen{}\setminus}{\{}{\}}{}{#1}		% for exclusion

\DeclareMathOperator{\dist}{dist}		% for distance
\DeclareMathOperator{\one}{\mathds{1}}		% for indicator
\DeclareMathOperator{\relint}{ri}		% for relative interior
\newcommand{\alt}[1]{#1'}		% for variant version
\newcommand{\altalt}[1]{#1''}		% for second variant
\newmacro{\dd}{\:d}		% for integration
\newcommand{\eps}{\varepsilon}		% for better epsilon
\newmacro{\const}{c}		% for generic constant
\newmacro{\Const}{\rho}		% for generic constant
\newmacro{\coefalt}{\mu}		% for generic coefficient
\usepackage{xparse}
\NewDocumentCommand{\coef}{O{$\lambda$}}{\debug{#1}}
\newmacro{\param}{\theta}		% for parameter
\newmacro{\params}{\Theta}		% for set of parameters

\newmacro{\pexp}{p}		% for first exponent
\newmacro{\qexp}{q}		% for second exponent
\newmacro{\rexp}{r}		% for third exponent

\newmacro{\radius}{r}

\newmacro{\beforestart}{0}		% for before start index
\newmacro{\start}{1}		% for start index
\newmacro{\afterstart}{2}		% for second index
\newmacro{\running}{\start,\afterstart,\dotsc}		% for running
\newmacro{\halfrunning}{1,3/2,2\dotsc}		% for running

\newmacro{\run}{t}		% for main sequence index
\newmacro{\runalt}{s}		% for variant index
\newmacro{\runaltalt}{\tau}		% for second variant
\newmacro{\nRuns}{T}		% for total number of runs
\newmacro{\runs}{\mathcal{\nRuns}}		% for set of indices

\newmacro{\state}{x}		% for main iterate
\newmacro{\statealt}{y}		% for variant state
\newmacro{\statealtalt}{z}		% for second variant

\newmacro{\tstart}{0}		% for time start
\newmacro{\timealt}{s}		% for dummy continuous time
\newmacro{\horizon}{T}		% for horizon

\newmacro{\traj}{x}		% for trajectory
\newmacro{\trajalt}{y}		% for variant trajectory
\newmacro{\trajaltalt}{z}		% for second variant

\newmacro{\flowmap}{\Theta}		% for (semi)flow map
\DeclarePairedDelimiterXPP{\flowof}[2]{\flowmap_{#1}}{(}{)}{}{#2}		% for flow

\newop{\Nash}{NE}		% for Nash equilibria
\newop{\CE}{CE}		% for correlated equilibria
\newop{\CCE}{CCE}		% for Hannan set
\newop{\NI}{NI}		% for Nikaido-Isoda function

\newop{\brep}{br}		% for best responses
\newop{\reg}{Reg}		% for regret
\newop{\preg}{\overline{Reg}}		% for pseudo-regret
\newop{\val}{val}		% for value function

\newcommand{\eq}{\sol}		% for Nash equilibrium
\newmacro{\play}{i}		% for player index
\newmacro{\playalt}{j}		% for variant player index
\newmacro{\playaltlalt}{k}		% for second variant
\newmacro{\nPlayers}{N}		% for number of players
\newmacro{\players}{\mathcal{\nPlayers}}		% for set of players

\newmacro{\pure}{\alpha}		% for pure strategy
\newmacro{\purealt}{\beta}		% for variant pure strategy
\newmacro{\purealtalt}{\gamma}		% for second variant
\newmacro{\nPures}{A}		% for number of pure strategies
\newmacro{\pures}{\mathcal{\nPures}}		% for set of pure strategies

\newmacro{\loss}{\ell}		% for loss function
\newmacro{\pay}{u}		% for payoff function
\newmacro{\payv}{v}		% for payoff vector
\newmacro{\pot}{f}		% for potential function

\newmacro{\game}{\mathcal{G}}		% for game
\newmacro{\gamefull}{\game(\players,\points,\pay)}		% for full game

\newmacro{\fingame}{\Gamma}		% for finite game
\newmacro{\fingamefull}{\Gamma(\players,\pures,\pay)}		% for full finite game

\newmacro{\gmat}{g}		% for metric tensor
\newmacro{\gdist}{\dist_{\gmat}}
\newmacro{\mfld}{M}		% for manifold
\newmacro{\form}{\omega}		% for generic form

\newmacro{\tvec}{z}		% for tangent vector
\newmacro{\uvec}{u}		% for unit vector

\newmacro{\ball}{\basin}		% for ball
\newmacro{\sphere}{\mathbb{S}}		% for sphere

\newmacro{\graph}{\mathcal{G}}
\newmacro{\vertices}{\mathcal{V}}
\newmacro{\edges}{\mathcal{E}}

\newmacro{\mat}{A}		% for generic matrix
\newmacro{\matalt}{c}		% for generic matrix
\newmacro{\hmat}{H}		% for Hessian matrix

\newop{\row}{row}		% for row space
\newop{\col}{col}		% for row space

\newmacro{\ones}{\mathbf{1}}		% for matrix of ones
\newmacro{\eye}{I}		% for identity matrix
\newmacro{\zer}{\mathbf{0}}		% for zero matrix

\DeclarePairedDelimiterXPP{\dnorm}[1]{}{\lVert}{\rVert}{_{\ast}}{#1}		% for dual norm

\DeclarePairedDelimiterXPP{\onenorm}[1]{}{\lVert}{\rVert}{_{1}}{#1}		% for dual norm
\DeclarePairedDelimiterXPP{\twonorm}[1]{}{\lVert}{\rVert}{_{2}}{#1}		% for dual norm
\DeclarePairedDelimiterXPP{\supnorm}[1]{}{\lVert}{\rVert}{_{\infty}}{#1}		% for dual norm

\DeclarePairedDelimiterX{\braket}[2]{\langle}{\rangle}{#1,#2}		% for brakets

\newmacro{\vecspace}{\mathcal{V}}		% for generic vector space
\newmacro{\subspace}{\mathcal{W}}		% for vector subspace

\newmacro{\coord}{i}		% for coordinate index
\newmacro{\coordalt}{j}		% for variant coordinate
\newmacro{\coordaltalt}{k}		% for second variant
\newmacro{\nCoords}{n}		% for number of coordinates
\newmacro{\dims}{\nCoords}		% for dimension
\newmacro{\vdim}{\nCoords}		% for dimension (legacy alias)

\newmacro{\pvec}{z}		% for primal vector
\newmacro{\pvecalt}{r}		% for primal vector
\newmacro{\bvec}{e}		% for basis vector
\newmacro{\bvecs}{\mathcal{E}}		% for basis vectors
\newmacro{\cvec}{b}     % for column vector
\newmacro{\cvecalt}{d}     % for column vector

\newmacro{\pspace}{\vecspace}		% for primal space
\newmacro{\dspace}{\vecspace^{\ast}}		% for dual space

\newmacro{\dvec}{\dpoint}		% for dual vector
\newmacro{\dbvec}{\eps}		% for dual basis vectors

\newmacro{\dpoint}{y}		% for generic dual point
\newmacro{\dpointalt}{\alt\dpoint}		% for variant dual point
\newmacro{\dpointaltalt}{\altalt\dpoint}		% for second variant
\newmacro{\dpoints}{\mathcal{Y}}		% for set of dual points

\newmacro{\dstate}{Y}		% for dual state
\newmacro{\dbase}{v}		% for dual base

\newop{\Opt}{Opt}		% for value of problem
\newop{\Sol}{Sol}		% for solution of problem
\newop{\gap}{Gap}		% for gap function
\newop{\orcl}{Or}		% for oracle

\newmacro{\tfun}{f}		% for test function
\newmacro{\obj}{f}		% for objective function
\newmacro{\objalt}{g}		% for variant objective (smooth etc.)
\newmacro{\sobj}{F}		% for stochastic objective

\newmacro{\gvec}{g}		% for gradient vector
\newmacro{\oper}{A}		% for operator
\newmacro{\vecfield}{v}		% for vector field (selection etc.)
\newcommand{\sol}[1][\point]{#1^{\ast}}		% for solution point (x by default)
\newmacro{\solvec}{\vecfield(\sol)}		% for vector at a solution
\newmacro{\solpay}{\eq[\payv]}		% for vector at a solution

\newmacro{\signal}{V}		% for signal
\newmacro{\step}{\gamma}		% for step-size
\newmacro{\learn}{\eta}		% for learning rate

\newmacro{\vbound}{G}		% for vector bound
\newmacro{\lips}{L}		% for Lipschitz modulus
\newmacro{\strong}{\mu}		% for strong convexity modulus
\newmacro{\smooth}{\beta}		% for strong smoothness modulus

\newop{\tspace}{T}		% for tangent space
\newop{\tcone}{TC}		% for tangent cone
\newop{\dcone}{\tcone^{\ast}}		% for dual cone
\newop{\ncone}{NC}		% for normal cone
\newop{\pcone}{PC}		% for polar cone
\newop{\hull}{\Delta}		% for hull

\newmacro{\cvx}{\mathcal{C}}		% for generic convex set
\newmacro{\subd}{\partial}		% for subdifferential
\newmacro{\minmax}{\mathcal{L}}		% for minmax objective

\newmacro{\minvar}{\point_{1}}		% for minimization variable
\newmacro{\minvaralt}{\alt\minvar}		% for variant minvar
\newmacro{\minvars}{\points_{1}}		% for set of minvars
\newmacro{\minsol}{\sol[\minvar]}		% for min solution

\newmacro{\maxvar}{\point_{2}}		% for maximization variable
\newmacro{\maxvaaltr}{\alt\maxvar}		% for variant maxvar
\newmacro{\maxvars}{\points_{2}}		% for set of maxvars
\newmacro{\maxsol}{\sol[\maxvar]}		% for max solution

\newop{\Eucl}{\Pi}		% for Euclidean projection
\newop{\logit}{\Lambda}		% for logit map
\newop{\dkl}{KL}		% for Kullback Leibler

\newmacro{\hreg}{h}		% for regularizer
\newmacro{\hconj}{\hreg^{\ast}}		% for regularizer
\newmacro{\breg}{D}		% for Bregman divergence
\newmacro{\mprox}{P}		% for Bregman prox-mapping
\newmacro{\mirror}{Q}		% for mirror map
\newmacro{\fench}{F}		% for Fenchel coupling
\newmacro{\depth}{H}		% for regularizer depth
\newmacro{\hstr}{K}		% for strong convexity constant
\newmacro{\hker}{\theta}		% for regularizer kernel

\newmacro{\proxdom}{\points_{\hreg}}		% for prox-domain
\newmacro{\proxdomi}{\points_{\hreg_{\play}}}		% for prox-domain of player i
\newmacro{\zone}{\mathbb{D}}		% for Bregman zone

\DeclarePairedDelimiterXPP{\proxof}[2]{\mprox_{#1}}{(}{)}{}{#2}		% for Bregman prox step

\newmacro{\point}{x}		% for generic point
\newmacro{\pointalt}{\alt\point}		% for variant point
\newmacro{\pointaltalt}{\altalt\point}		% for second variant
\newmacro{\points}{\mathcal{X}}		% for set of points
\newmacro{\intpoints}{\relint\points}		%for point set interior

\newmacro{\base}{p}		% for reference point
\newmacro{\basealt}{q}		% for variant reference point
\newmacro{\basealtalt}{u}		% for second variant

\newmacro{\open}{\mathcal{U}}		% for open sets
\newmacro{\closed}{\mathcal{C}}		% for closed sets
\newmacro{\cpt}{\mathcal{K}}		% for compact sets
\newmacro{\nhd}{\mathcal{U}}		% for neighborhoods

\newop{\ex}{\mathbb{E}}		% for expectations
\newop{\prob}{\mathbb{P}}		% for probability
\newop{\Var}{Var}		% for variance
\newop{\simplex}{\hull}		% for simplices

\DeclarePairedDelimiterXPP{\exof}[1]{\ex}{[}{]}{}{%		% for conditional expectations
 #1}

\DeclarePairedDelimiterXPP{\probof}[1]{\prob}{(}{)}{}{%		% for conditional probabilities
 #1}

\DeclarePairedDelimiterXPP{\oneof}[1]{\one}{\{}{\}}{}{%		% for conditional expectations
 #1}

\newmacro{\sample}{\omega}		% for sample
\newmacro{\samples}{\Omega}		% for set of samples

\newmacro{\seed}{\theta}		% for seed
\newmacro{\seeds}{\Theta}		% for seed space

\newmacro{\filter}{\mathcal{F}}		% for filtration
\newmacro{\probspace}{(\samples,\filter,\prob)}		% for probability space

\newmacro{\history}{\mathcal{H}}		% for filtrations

\newmacro{\event}{E}       % for event
\newmacro{\eventalt}{H}       % for variant event

\newmacro{\mean}{\mu}		% for mean of distribution
\newmacro{\sdev}{\sigma}		% for mean of distribution
\newmacro{\variance}{\sdev^{2}}		% for mean of distribution

\newmacro{\proper}{\tau}		% for proper time
\newmacro{\error}{Z}		% for error
\newmacro{\noise}{U}		% for noise
\newmacro{\bias}{b}		% for bias
\newmacro{\brown}{W}		% for Wiener process

\newmacro{\serror}{\theta}		% for scalar error
\newmacro{\snoise}{\xi}		% for scalar noise
\newmacro{\sbias}{\psi}		% for scalar bias

\newmacro{\sbound}{M}		% for signal bound
\newmacro{\bbound}{B}		% for bias bound
\newmacro{\noisepar}{\sdev}		% for noise parameter
\newmacro{\noisevar}{\variance}		% for noise variance

\addauthor[Generic Author]{GA}{DarkGreen}

\addauthor[Pan]{PM}{MediumBlue}

\newmacro{\pdist}{P}		% for seed law

\newcommand{\EMAIL}[1]{{\href{mailto:#1}{#1}}}

\newcommand{\ourtitle}{Chaos persists in large-scale multi-agent learning\\ despite adaptive  learning rates}
\title
[\sc Chaos in adaptive multi-agent learning]
{\LARGE\textbf{\textsc{\ourtitle}}}

\author
[E.~V.~Vlatakis-Gkaragkounis]
{Emmanouil V. Vlatakis-Gkaragkounis$^{\S}$}
\address{$^{\S}$\,%
Department of Computer Sciences, University of Wisconsin--Madison, 1210 West Dayton Street, Madison, WI 53706, USA.
\EMAIL{vlatakis@wisc.edu}
}

\author
[K.~Metaxas]
{Konstantinos Metaxas$^{\dagger}$}
\address{$^{\dagger}$\,%
Department of Mechanical Engineering, Massachusetts Institute of Technology, 77 Massachusetts Ave., Cambridge, MA 02139, USA.
\EMAIL{metaxas@mit.edu}
}

\author
[L.~Flokas]
{Lampros Flokas$^{\ast}$}
\address{$^{\ast}$\,%
Department of Computer Science, Columbia University, 500 West 120th Street, Mail Code 0401, New York, NY 10027, USA.
\EMAIL{lamflokas@cs.columbia.edu}
}

\author
[G.~Piliouras]
{Georgios Piliouras$^{\diamond,\star}$}
\address{$^{\diamond}$\,%
Singapore University of Technology and Design, 8 Somapah Road, Singapore 487372, Singapore.}
\address{$^{\star}$\,%
Google DeepMind, London, United Kingdom.
\EMAIL{georgios.piliouras@gmail.com}
}


\begin{document}
\raggedbottom

%\addtocontents{toc}{\protect\setcounter{tocdepth}{0}}
% Title page for title and abstract only.
\allowdisplaybreaks		% for breaking long displays
\acresetall		% for resetting acros

%----------------------------------------------------------------------
%%% ABSTRACT
%----------------------------------------------------------------------
\begin{abstract}
%----------------------------------------------------------------------
%%% ABSTRACT
%----------------------------------------------------------------------
% !TEX root = ../main_new.tex
%
%
% What is the (important) problem
% MWU, game dynamics learning rate, chaotic behavior  
% Existing approaches fail
% Our solutions solves the problem
%Chaotic behaviors are observed in real life games even in conventionally easy classes like congestion games \cite{arthur1994inductive}. This behavior has been formally verified for the widely popular meta-algorithm Multiplicative Weight Updates (MWU) when agents are using high learning rates. From a behavioral perspective, there is still a large gap between these theoretical results and real life agents that can adapt their meta-behavior dynamically in response to history of the game. In this paper we aim to close this gap by extending prior results to settings where agents can dynamically change their learning rate based on a summary of the historical payoffs. At a technical level, to the best of our knowledge our results are the first to establish the existence of chaos in MWU dynamics in dynamic learning rate environments.
%with the goal of having optimal performance in the presence of stationary opponents as well a
%where the agents  can  
%learn quickly when they perform poorly and are more cautious when they are performing well.  

Multi-agent learning is intrinsically harder, more unstable and unpredictable than single agent optimization. For this reason, numerous specialized heuristics and techniques have been designed towards the goal of
achieving convergence to equilibria in self-play. One such celebrated approach is the use of  dynamically adaptive learning rates. 
Although such techniques are known to allow for
improved convergence guarantees in small 
%(two agent, two strategy)
games, it has been much harder to analyze them in more relevant settings with large populations of agents. These settings are particularly hard as recent work has established that learning with fixed  rates will become chaotic given large enough populations.
%~\cite{chotibut2019route,bielawski2021follow}.
In this work, we rigorously show that topological chaos persists in large population congestion games
despite using adaptive learning rates even for the ubiquitous Multiplicative Weight Updates (MWU) algorithm, even 
in the presence of only two strategies. The non-autonomous nature of the system under study calls for fundamentally different techniques from those used in the autonomous case, with our approach focusing on the geometric properties of the dynamics, including the investigation of invariant sets, volume expansion, and turbulent sets. We close the study discussing the attractors of the system and the asymptotic relationship between the adaptive and standard MWU algorithms. The results established in this work motivate the extension of rigorous chaos analysis to other non-autonomous multi-agent learning settings as well.

\end{abstract}
\maketitle

\clearpage
{\normalsize\tableofcontents}
\clearpage

%----------------------------------------------------------------------
%%% INTRODUCTION
%----------------------------------------------------------------------
\section{Introduction}
\label{sec:introduction}
%----------------------------------------------------------------------
%%% INTRODUCTION
%----------------------------------------------------------------------
% !TEX root = ../main_new.tex

%This is the introduction, and this is a test citation of \citet{Nas51}.

% Here is a problem. How do MAL dynamics with variable learning rates behave in large games? How can we analyze them?
% It is an interesting problem.
% There is a lot of prior work on the ML side showing that at least in small games there are provable advantages when allowing for agents to adapt their learning rates based on cues coming from the state of the game.
% It is an unsolved problem. No positive results are known. Arguing about negative results is known to require techniques which are not applicable in our case.
% Here is our solution 
% Our solution works

Arguably one of the most thorny problems on the intersection of learning and games is the development of simple and practical algorithms that converge to Nash equilibria. The problem in its full generality is known to be intractable~\cite{daskalakis2006complexity,hart2003uncoupled}, however, recent developments in Machine Learning have led to a revived interest in the problem even in special classes of games. Unfortunately, the additional scrutiny has only helped crystallize the severity of the problem at hand through a diverse set of non-convergence, instability    
 results %for numerous standard learning dynamics 
 even in well motivated special classes of games~\cite{mertikopoulos2018cycles,bailey2019fast,flokas2020no,andrade2021learning,giannou21survival,letcher2021impossibility,bailey2018multiplicative,hsieh2021limits,bailey2020finitecolt,balduzzi2018mechanics}. Worse yet, standard online learning meta-algorithms, such as Multiplicative Weights Updates (MWU), have been shown to exhibit chaos~\cite{palaiopanos2017multiplicative,cheung2019vortices,cheung2020chaos,cheung2021chaos,bielawski2021follow,chotibut2019route, metaxas, bielawski2024memory}, which has been shown to be rather common in more complex games, e.g., as we increase the number of agents  ~\cite{GallaFarmer_PNAS2013,sanders2018prevalence,bielawski2021follow,chotibut2019route} and even in heterogenous settings \cite{bielawski2025heterogeneity}.
 Chaotic behavior has also been observed in auctions \cite{cheung3} and in dynamics induced by fictitious play learning rules \cite{sparrow,vanStrien}, as well as in discretized replicator dynamics \cite{mukhopadhyay,pandit,vilone}.
 Given this proliferation of negative results,
 %raises a tantalizing search for new ideas
 %Is there a promising path around such negative results? 
 %Given 
  what other approaches are left to explore?

 An interesting hint can be found in some of the earliest AI work on the subject. \citep{singh2000nash} established arguably one of the first non-convergence results in the area, showing that gradient dynamics do not suffice to achieve point-wise convergence even in the trivial case of two agent, two strategy games. On the positive side, when those dynamics are non-convergent, time-average convergence results can be established.  Building up on their work, \citep{bowling2002multiagent} showed that a modification of these standard dynamics where the agents can dynamically update their step-size based on payoff cues from their environment suffices to stabilize these dynamics in all $2\times 2$ games. 
 Informally, the specific heuristic, Win or Learn Fast (WoLF), has the agents increase their learning rate when they are "losing" in an effort to escape from a non-promising region of the state space. Unfortunately, WoLF requires that each agent has knowledge of a Nash equilibrium strategy for themselves, which means it can only be applied in small games. On the positive side, this first result has led to several other similar heuristics, with promising results in small games~\cite{banerjee2003adaptive,bowling2004convergence,abdallah2008multiagent,kaisers2009evolutionary,leonardos2022exploration,bloembergen2015evolutionary}. While staying in the realm of small games, such heuristics have been shown to robustly improve the stability of other popular learning heuristics such as Proximal Policy Optimization \cite{schulman2017proximal,ratcliffe2019win}.
 On the negative side, even for slightly larger games most positive results are largely empirical in nature and effectively very little is known about the behavior of such techniques in games with many agents.
 %A further benefit of this approach is that is well supported by behavioral game theory \cite{camerer1998experience} and thus insights into such dynamics might also 
These works raise our key motivating question: 
\begin{quote}
\textit{Does the simple idea of judiciously increasing the learning rate to ``escape"  barren regions of the state space scale to games with a large number of agents? If not, what type of formal instability results can be established?}
\end{quote}

{\bf Our approach.} We focus on  one of the best-studied classes of large population games, non-atomic congestion games~\cite{roughgarden2002bad,Nisan:2007:AGT:1296179}
%\cite{roughgarden2004bounding}
where the agents learn using the ubiquitous MWU update rule~\cite{Arora05themultiplicative,freund1999adaptive,Littlestone:1994:WMA:184036.184040}. Furthermore, given the recent works of \cite{chotibut2019route,bielawski2021follow}, we focus on games with exactly two strategies/routes and linear cost functions, since such settings are already sufficiently hard for dynamics. Specifically, given fixed, non-adaptive learning rates, MWU dynamics typically will bifurcate to chaotic behavior for a large enough population size. Critically, however, these previous results are based on the celebrated "Period-3 implies chaos" methodology pioneered by Li--Yorke \cite{liyorke}, which is only applicable in autonomous, i.e., time-invariant systems. In contrast, understanding the emergent behavior in our case will require totally different techniques. 

Our model will allow for a wide range of dynamically adaptive learning rates. In particular, instead of having only two learning rates, fast and slow, as, e.g., in \citep{bowling2002multiagent}, we will allow for a continuous range of learning rates. Moreover,
the decision to increase/decrease the learning rate will be driven by a regret-like measure~\cite{shalev2011online}. Intuitively, each agent compares the historical time-average performance of both routes available to them.  If both routes appear very similar to each other then the agents favor large learning rates so as to reach sufficiently different configurations where they can hopefully find informative payoff signals to exploit. On the other hand, if one route is significantly better than the other then the agents will 
 %act more cautiously and 
take smaller but non-vanishing step sizes so as to take advantage of such opportunities without overcorrecting.

\textbf{Our main result.} We prove that our class of dynamics  exhibits Li--Yorke chaotic behavior, implying the existence of an uncountable set of initial conditions 
that become ``scrambled" by the dynamics (Theorem~\ref{theorem:dynamic-li-yorke}). Formally, for any two initial conditions from this set, the trajectories of the game dynamics come arbitrarily close to each other and move away from each other infinitely often.

{\bf Our techniques.} Despite the intense recent interest in developing formal arguments about chaos in game settings~\cite{bielawski2021follow,chotibut2019route,leonardos2021dynamical,piliouras2022multi}
%\cite{bielawski2021follow,chotibut2019route,leonardos2021dynamical,cheung2021learningIJCAI,piliouras2022multi},  
proving chaos in non-autonomous dynamical systems is a challenging endeavor, as the departure from one-dimensional autonomous systems gives up many established equivalences among different notions of chaos, such as volume expansion, topological entropy, and arbitrarily long periodic points.

% Let's not describe things that do not work. 
One possible approach to addressing this challenge would be to first establish that the sequence of employed maps converges uniformly to a fixed chaotic map, and then transfer the resulting scrambled set of the limit system to the non-autonomous one.  
%
%However, a key obstacle in implementing this proof strategy is that the MWU map can never be uniformly convergent over the entire strategy space $\mathcal{X}$, as at the boundary of the pure action space $\mathrm{bd}(\mathcal{X})$, the map has trivial fixed points, and our adaptive update rule will converge to different behavior than in the interior of the strategy space $\mathrm{int}(\mathcal{X})$ .

This idea has been formalized in the theory of non-autonomous discrete
dynamical systems~\cite{kolyada1996topological}: for a sequence of surjective
interval maps $f_n$ converging uniformly to $f$, Li--Yorke chaos of $f$ is
inherited by the non-autonomous system~\cite{stefankova2016inheriting}
(see also~\cite{canovas2011li}), although other notions of chaos need not
transfer~\cite{dvorakova2012chaos}. These results do not apply directly here:
the learning rate $a_n$ depends on the whole history of the orbit, and hence on
$x_0$, so the iteration is not driven by a fixed sequence of maps, and the
convergence $a_n \to a^\star$ is uniform only on compact subintervals of
$(0,1)$, not on $[0,1]$.

Although such properties are not true for our system, it suffices to use a slightly relaxed version of them. We
 establish uniform convergence guarantees for initialization sets that are confined to the interior of the strategy space (Lemma~\ref{lemma:dynamic-strong-convergence}), while tolerating the discrepancy between the varying-step and the corresponding MWU map in the limit. To confirm the existence of an uncountable scrambled set, we propose a novel connection between the behavior of MWU maps with varying learning rates and symbolic dynamics. Essentially, our proof strategy entails assigning every binary sequence to a different initialization such that scrambled binary strings correspond to scrambled initializations of our non-autonomous dynamical system (Proof of Theorem~\ref{theorem:dynamic-li-yorke}).

The first part of our analysis (Section~\ref{sec:fixed}) involves  enhancing previously existing results for the case of fixed learning rate dynamics. Specifically, we present an explicit construction of a \emph{perpetual set} $\mathcal{F}(a)$, i.e., a forward invariant set where our dynamical system is surjective, $f_{a}(\mathcal{F}(a))=\mathcal{F}(a)$, where $f_{a}$ represents the MWU system with fixed learning rate $a$. (Lemma~\ref{lemma:perpetual}). Next we further show that, even in the presence of agent volatility due to their high learning rates, if their initial strategies lie within the interior of the strategy space, their strategies will eventually be absorbed into the perpetual set we have constructed (Lemma~\ref{lemma:absorbption}).

In the case of fixed learning rate, we are able to quantify an even stronger volume expansion result. Specifically, we show that the image of any arbitrarily small neighborhood of the mixed equilibrium of our game will converge exactly to the perpetual set we have constructed (Theorem~\ref{thm:volume-expansion}).

% Turning our attention to the adaptively dynamic learning rate, we demonstrate that the construction of the perpetual set $\mathcal{F}(a)$ allows us to show that, despite the volatility of players due to potentially high learning rates, there exists a set $\Delta$ - the closure of all perpetual sets $\mathcal{F}(a_n)$ - to which the dynamics will eventually converge, provided the initialization set lies within the interior of the solution space (Lemma~\ref{?},\ref{?}). Furthermore, utilizing this construction, we show that the average regret-like concepts converge for nearly all initial conditions (Theorem/Lemma ??). 

Turning our attention to the adaptively dynamic learning rate (Section~\ref{sec:dynamic}), we show that thanks to $\mathcal{F}(a)$, while learning rate may vary over iterations and initializations, there exists a set $\Delta$ \textendash the closure of all perpetual sets $\mathcal{F}(a_n)$ \textendash to which our dynamics eventually will be absorbed (Lemmas~\ref{lemma:forward-dynamic},\ref{lemma:absorption-dynamic}), provided the initialization set lies in the interior of solution space. Leveraging this property, we show that despite the volatility due to agents employing potentially high learning rates, our average regret-like concepts converge for nearly all initial conditions (Lemmas \ref{ref:average},\ref{ref:uniform},\ref{ref:cesaro}). It is noteworthy that this result requires a substantial degree of additional  technicality in contrast to its analogue for the fixed rate regime \cite{chotibut2019route}. 

%Intuitively our last preliminary step for our re, it is necessary to show that
Before developing our machinery for the symbolic dynamics reduction in Section~\ref{sec:turbulence}, our proof strategy includes an additional step of demonstrating that our system does not collapse early in any fixed point or subspace of strategies that lacks volume expansion, as this would eliminate chaos. To achieve this, we rely on Theorem~\ref{theorem:dynamic-volume-expansion} which guarantees that almost all neighborhoods of mixed equilibrium will avoid this scenario by  volume expanding and approximately covering $\mathcal{F}(\displaystyle\lim_{n\to\infty}a_n)$.
Putting everything together our main result about chaotic limit behavior follows. Finally, in Section \ref{sec:sec 6} we discuss the attractors of the system and the asymptotic relationship between the adaptive and standard MWU algorithms. 

\section{Problem setup and preliminaries}
\label{sec:prelims}
%----------------------------------------------------------------------
%%% INTRODUCTION
%----------------------------------------------------------------------
% !TEX root = ../main_new.tex

% What should go in the perliminaries?
% What is an autonomous dynamical system
% What is a non-autonomous dynamical system
% Definition of Li Yorke chaos for non-autonomous dynamical systems
% Period three result
% Congestion games
% MWU map
% Our dynamics

% {\color{red}\textbf{Notation} Maybe describe some notation here?}

\subsection{Dynamical systems \& Li--Yorke chaos}
% What is an (autonomous/non-autonomous) dynamical system? What is a trajectory of an initialization?
A discrete autonomous $m$-dimensional dynamical system is described by an equation $x_{n+1} = f(x_n)$ for a function $f: \mathbb{R}^m \to \mathbb{R}^m$. In contrast a discrete $m$-dimensional non-autonomous dynamical system is described by an equation $x_{n+1} = f(x_n, n)$ for a function $f: \mathbb{R}^m \times \mathbb{N} \to \mathbb{R}^m$. In both of these cases we can view the $x_n(x_0)$ iterate as a function of the initialization $x_0$. We thus term the sequence $\{x_n(x_0) \}_{n \in \mathbb{N}}$ the orbit of the dynamical system of $x_0$.

\begin{definition}
 We call two initializations $p,q$ of a dynamical system \emph{scrambled} if $\limsup_{n \to \infty} \| x_n(p) - x_n(q)\| >0$ and $\liminf_{n \to \infty} \| x_n(p) - x_n(q)\| = 0$. Intuitively the orbits of $p,q$ continuously alternate between moving apart and arbitrarily close respectively. A set $S$ is scrambled if every pair $p, q \in S$ is scrambled. 
\end{definition}

Based on the concept of a scrambled set, we can define the topological notion of Li--Yorke chaos for both autonomous and non-autonomous dynamical systems.

\begin{definition}
An autonomous/non-autonomous dynamical system is called Li--Yorke chaotic if there is an uncountably infinite set $S$ that is scrambled. 
\end{definition}

A famous result for the case of one dimensional autonomous dynamical systems provides a simple and easily verified sufficient condition for chaos based on periodic orbits:
\begin{theorem}[\cite{liyorke}]\label{thm:Li-Yorke}
Let $J$ be an interval and $f: J \to J$ be a continuous map of a dynamical system. If $f$ has a period-3 orbit, then it is Li--Yorke chaotic. 
\end{theorem}
\subsection{Linear congestion games} In this work we consider the case of non-atomic two strategy linear congestion games. In a non-atomic congestion game, there is a continuum of agents each of which controls an infinitesimal fraction of the total flow $N$. In each iteration, a fraction of agents $x$ choose the first action and $1-x$ the second one, suffering costs $c_1(x) = \gamma N x$ and $c_2(1-x) = \delta_{\mathrm{cost}} N (1-x)$ respectively where $\gamma, \delta_{\mathrm{cost}} > 0$ are the linear cost coefficients of each action respectively. 

\subsection{Multiplicative Weights Update}
The dynamics of the MWU meta-algorithm correspond to an autonomous dynamical system. Given the common learning rate of all agents $\eta > 0$, the MWU update is as follows
\begin{align*}
    x_{n+1} = \frac{x_n e^{-\eta c_1(x_n)}}{x_n e^{-\eta c_1(x_n)} + (1-x_n) e^{-\eta c_2(1-x_n)}}
            = \frac{x_n}{x_n + (1-x_n) e^{-\eta (c_2(1-x_n) - c_1(x_n))}}
\end{align*}

We can further simplify the MWU update rule by making the following changes of variables
\begin{align*}
    a = \frac{N\eta}{\delta_{\mathrm{cost}} + \gamma} && b = \frac{\delta_{\mathrm{cost}}}{\delta_{\mathrm{cost}} + \gamma}
\end{align*}
where $a$ corresponds to a normalized learning rate that accounts for the scale of $c_1$ and $c_2$ and $(b, 1-b)$ is the Nash equilibrium of the game. The simplified update rule can be expressed via a parametric map $f$ over $a$ and $b$ that captures the whole class of games we are interested in
{\small
\begin{equation}\label{eq:transform}
    x_{n+1} = f(x_n, a, b) = \frac{x_n}{x_n + (1-x_{n}) \exp(a (x_n -b))}.
\end{equation}
}
When $a,b$  are fixed we may skip them from the notation of MWU map and use just $f(x)$.

\subsection{Dynamic Learning Rate Model}
\label{sec:dynamic_rate}
Here we extend the MWU meta-algorithm to adaptive learning rates $\eta$. The basis of our adapting dynamics is a notion of pseudo-regret that the agents suffer when choosing the first action over the second one. The definition captures the average difference of costs between the two actions weighted by the learning rate at each iteration $\eta_i$:
\begin{equation*}
    r_{n+1} = \frac{1}{n+1} \sum_{i=0}^n \eta_i  [c_1(x_i) - c_2(1-x_i)].
\end{equation*}
Based on this calculation of pseudo-regret, the agents may seek to adapt their learning rates $\eta_i$ in order to balance between exploration and exploitation. When $r_n$ is big then it is clear that one action dominates over the other so the agents  seek to move in the direction of improving costs but hopefully not too aggressively so as not to overshoot, akin to a gradient-like dynamics. 
%slow down learning to make extended use of this cost gap for more iterations.
In contrast, when $|r_n|$ is close to zero the average penalty of switching between actions is small since they suffer the same weighted cost in average. Thus there is an opportunity for the agents to increase their learning rate to encourage exploration while the cost differences are low.  To hedge between exploration and exploitation as above, agents can pick appropriate function $h:\mathbb{R}\to \mathbb{R}^+$ that peaks at $0$ and setting $\eta_n = h(r_n)$.

This dynamics of adapting $\eta_i$ to the pseudo-regret can be implemented as a 2-dimensional non-autonomous dynamical system with $x_n$ and $r_n$ as its state variables:
\begin{align}\label{eq:old-update}
    x_{n+1} = f\left(x_n, \frac{Nh(r_n)}{\gamma+\delta_{\mathrm{cost}}}, b \right) \quad
    r_{n+1} = \frac{nr_n + h(r_n)(c_1(x_n) - c_2(1-x_n))}{n+1}.
\end{align}
The system is non-autonomous in the 2-dimensional space as the update rule of $r_n$ depends on $n$. We can always view $d$-dimensional non-autonomous systems as autonomous systems in $d+1$ dimensions by treating $n$ as an additional dimension. However, this does not really simplify things as \cref{thm:Li-Yorke} does not apply beyond a single dimension.

For simplicity, we rewrite the dynamics in terms of the normalized learning rates $a_i$. Let $a_i$ be the normalized learning rates corresponding to $\eta_i$ based on \cref{eq:transform}. We can then rewrite $r_{n+1}$ as
\begin{equation*}
    r_{n+1} = \frac{1}{n+1} \sum_{i=0}^n a_i  [x_i-b].
\end{equation*}
Similarly, given that $\eta_i$ and $\alpha_i$ are equivalent up to game specific scale parameters, we can always find a $g$ such that $a_n = g(r_n)$ . Instead of working in terms of the cumbersome dynamical system update rule \cref{eq:old-update}, we will be performing our analysis in terms of these simpler equations
\begin{align*}
    x_{n+1} = f(x_n, a_n, b) &&
    a_{n+1} = g\left( \frac{\sum_{i=0}^n a_i(x_i - b)}{n+1}\right).
\end{align*}
In the following sections we will make some assumptions on $g$. We will choose $g$ to be a continuous function bounded in $[a_{\min}, a_{\max}] \subset \mathbb{R}^+$. We will also assume that the initialization $a_0$ is also a constant in $[a_{\min}, a_{\max}]$.

Throughout this work we will think of $x_n$ and $a_n$ as functions of the initialization $x_0, a_0$.
Thus, the complete notation of the $n$-th iterate would be $x_n(x_0,a_0)$ and $a_{n}(x_0,a_0)$, however, when it does not hinder understanding we will drop the explicit description of all the dependencies. 

% N [\delta_{\mathrm{cost}} (1- x) - \gamma x] = N [\delta_{\mathrm{cost}} - (\delta_{\mathrm{cost}} + \gamma) x]  = N/(\delta_{\mathrm{cost}}+\gamma) [p - x]
% (1-\epsilon)^t = exp(t \ln(1-e)

\section{Refinements in fixed learning rates}
\label{sec:fixed}

%\begin{figure}[h!]
%\centering\includegraphics[width=0.7\textwidth]{Figures/mwu.pdf} 
%    \caption{The order of 
%    {\color{orange}$ f_{min}(a) $},
%    {\color{red}$ x_{\max}(a)$},
%    {\color{black}$ b $},
%    {\color{magenta}$ x_{\min}(a) $},
%    {\color{OliveGreen}$ f_{max}(a)$} in $x$-axis.
%    }
%    \label{fig:my_label}
%\end{figure}

We start our analysis by providing a refined analysis in the regime of fixed learning rate as in \cite{chotibut2019route}. The proofs of the technical results of the present section are deferred to Appendix \ref{app:fixed}. In the following we will examine conditions of such systems under the assumption of $a>4$. Let us study the local minima and maxima of $f$:
\begin{equation*}
    x_{\min/\max}(a) = \frac{1}{2} \pm \sqrt{\frac{1}{4} - \frac{1}{a}} %,\quad x_{\min}(a) = \frac{1}{2} + \sqrt{\frac{1}{4} - \frac{1}{a}}}
\end{equation*} 
Correspondingly we will denote $f_{\max}(a)=f(x_{\max}(a),a,b)$ and $f_{\min}(a)=f(x_{\min}(a),a,b)$. Our first observation is connected with the order of these values for high enough learning rate:
\Needspace{27\baselineskip}
\begin{lemma}\label{lemma:structure}
For every $b \in (0,1)$, there is a $a_b$ such that for all $a> a_b$
\begin{equation*}
    f_{\min}(a) < x_{\max}(a)< b < x_{\min}(a) < f_{\max}(a)
\end{equation*}
and $f_{\min}(a)$ is decreasing and  $f_{\max}(a)$ is increasing.
\end{lemma}

\begin{figure}[H]
    \centering
    \includegraphics[width=0.7\linewidth]{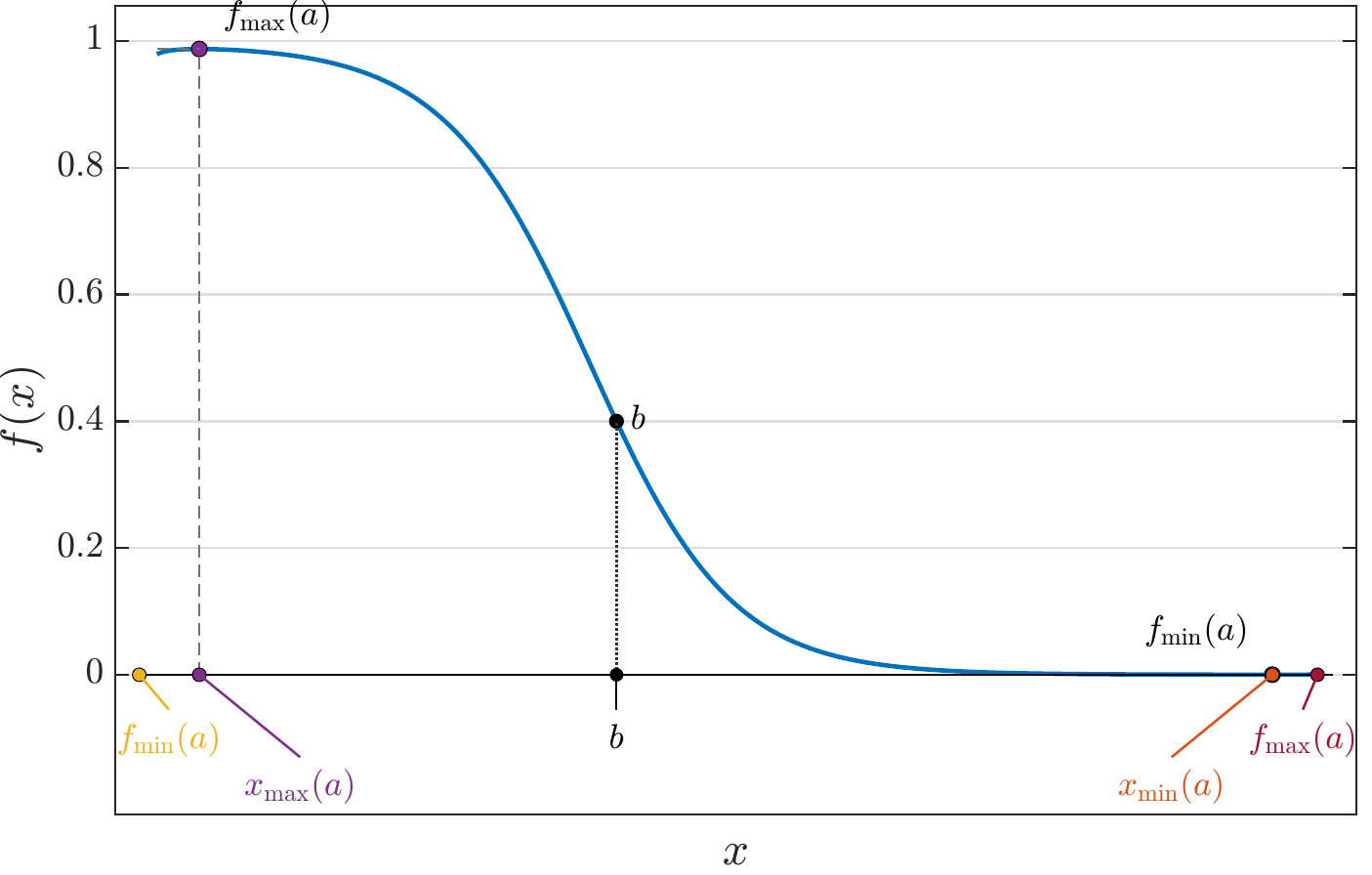}
    \caption{The order of $ f_{\min}(a) $, $ x_{\max}(a)$, $ x_{\min}(a) $ and $ f_{max}(a)$ in the $x$-axis.}
    \label{fig:my_label}
\end{figure}

% \begin{figure}[!ht]
%     \centering
%     \includegraphics[width=0.6\textwidth]{Figures/mwu.pdf}
%     \caption{The order of 
%     {\color{orange}$ f_{min}(a) $},
%     {\color{red}$ x_{\max}(a)$},
%     {\color{black}$ b $},
%     {\color{magenta}$ x_{\min}(a) $},
%     {\color{OliveGreen}$ f_{max}(a)$} in $x$-axis.
%     }
%     \label{fig:my_label}
% \end{figure}

Illustratively, the order described in Lemma \ref{lemma:structure} is shown in Fig. \ref{fig:my_label}. Additionally, by construction, we can show that the interval 
$\mathcal{F}(a)=[f_{min}(a) , f_{max}(a)]$ is a forward invariant set for our dynamical system. 
In other words, if the MWU map with fixed learning rate $a$ starts in a state that is within $\mathcal{F}(a)$, it will remain in that set for all future times.
\begin{lemma}\label{lemma:forward}
For every $b \in (0,1)$ there is a $s_b$ such that  $\mathcal{F}(a)$ is forward invariant for all $a> s_b$, i.e. $$
    x \in \mathcal{F}(a) \Rightarrow f(x, a, b) \in \mathcal{F}(a).$$
\end{lemma}
We can also prove that $f$ is surjective on $\mathcal{F}(a)$ for high enough learning rates.
\begin{lemma}\label{lemma:perpetual}
For $a > a_b$, $f$ is surjective on $\mathcal{F}(a)$, i.e. 
$
    f(\mathcal{F}(a), a, b) = \mathcal{F}(a).
$
\end{lemma}
When $\mathcal{F}(a)$ is both surjective and forward invariant as is the case for $a> a_b, s_b$, we will call $\mathcal{F}(a)$ a \emph{perpetual} set. Actually, $\mathcal{F}(a)$ is also an absorbing set for all $x\in(0,1)$, in the following sense:

% \begin{proof}
% We already know that the right hand side is a super set of the left hand side by \Cref{lemma:forward}. But by \Cref{lemma:structure} we have that $x_{\max}(a)$ and $x_{\min}(a)$ belong to $[f_{\min}(a), f_{\max}(a)]$. So the left hand is a super set of the left hand side. Thus the result follows.
% \end{proof}  

\begin{lemma}\label{lemma:absorbption}
Let $a> s_b$,  $
      x_0 \in [\gamma, \zeta] \Rightarrow \exists n_0 : \ \forall n \geq n_0 \text{ s.t } \ x_{n} \in \mathcal{F}(a)
$,  for every $[\gamma, \zeta] \subset (0,1)$.
\end{lemma}

We now turn to the chaotic properties of $f$. Given that period-3 orbits do not carry over to the dynamic learning rate case, we choose to study an alternative property of chaotic maps, namely volume expansion. In autonomous maps, the existence of period-3 orbits together with \cref{thm:Li-Yorke} imply that there is an initialization set whose volume under the dynamics expands quickly. Intuitively, the more volume this initialization set covers the more chaotic the dynamic is. While \cref{thm:Li-Yorke} guarantees the existence of volume expansion, it does not quantify how much volume it eventually covers. We prove that any interval around $b$ is sufficient to eventually cover $\mathcal{F}(a)$:

\begin{theorem}\label{thm:volume-expansion}
For every $b \in (0,1)$, there is a $v_b$ such that for all $a > v_b$ and  any interval $[\gamma, \zeta]$  : $  \{ b\} \subset [\gamma, \zeta] \subset (0,1) $, it holds that
\begin{equation*}
 \exists n_0 : \ \forall n \geq n_0 \quad f^n([\gamma, \zeta], a, b) = \mathcal{F}(a).
\end{equation*}
\end{theorem}

The proof of this theorem, which notably has not been established in any prior work, relies on a novel argument that is based on the monotonicity of $f$ on $\mathcal{D}=[x_{\max}(a), x_{\min}(a)]$, the lack of period-2 trajectories in $\mathcal{D}$ and the instability of fixed point $x=b$.

%To sum up, in this section, we manage to quantify  the eventual image of any neighborhood around $b$, which notably has not been established in any prior work.
The key takeaway for the next section is that we can focus our efforts on analyzing the behavior of the non autonomous system in the interior of $(0,1)$ with special attention to neighborhoods of $b$ that alone can exhibit chaotic behavior. This is especially important because as we will see our notion of pseudo-regret converges uniformly in closed intervals in the interior of $(0,1)$ but not on $[0,1]$.

\section{Fundamental properties of the adaptive learning dynamics}
\label{sec:dynamic}

Turning our attention to the dynamic learning rate setting, we will refer to $x_n(x_0)$  as the $n$-th iterate given the initialization $x_0$, and $a_{n}(x_0)$ the learning rate at the same iteration, respectively. In this section we prove some fundamental properties of the adaptive learning dynamics associated with the existence of a forward \& absorbing set, the convergence of the Ces\`{a}ro means, even when the dynamics is complex and the volume expansion. Particularly the latter will prove useful in the section where the learning dynamics under study will be shown to be Li--Yorke chaotic.

\subsection{Forward invariant \& absorbing set in dynamic learning rates.}
Leveraging the existence of the perpetual set $\mathcal{F}(a)$ in the fixed rate case, we can construct a forward invariant absorbing set. This set will be crucial in proving the Li--Yorke chaotic behavior for the non-autonomous case. Intuitively, although the learning rate is varying both among different initializations and iterations, there exists a set which corresponds to the closure of all perpetual sets $\{\mathcal{F}(a_n)\}$, to which our dynamics are always absorbed. This motivates the definition of set $\Delta$ as 
$$\Delta = \left[\min_{a \in [a_{\min},a_{\max}]}f_{\min}(a) , \max_{a \in [a_{\min},a_{\max}]}f_{\max}(a)\right]$$ 
where $a_{\min}$ and $a_{\max}$ are the $\inf_{r\in\mathbb{R}}g(r)$ and 
$\sup_{r\in\mathbb{R}}g(r)$, respectively. 

\begin{lemma}[Forward Invariance Property]\label{lemma:forward-dynamic}
For all $a_{\min} > s_b$ we get that  $\Delta$ is forward invariant, i.e.
\begin{align*}
x_n(x_0) \in \Delta \implies x_{n+1}(x_0) \in \Delta.
\end{align*}
\end{lemma}
\begin{proof}
We have three cases to consider. The first one 
\begin{equation*}
 x_n(x_0) \in [f_{\min}(a_n(x_0)), f_{\max}(a_n(x_0))] \implies x_{n+1}(x_0) \in [f_{\min}(a_n(x_0)), f_{\max}(a_n(x_0))]   
\end{equation*}
and the result follows trivially because of the perpetual set $\mathcal{F}(a_n(x_0))$. For the second case we get
\begin{equation*}
    x_n(x_0) \in \left[\min_{a \in [a_{\min},a_{\max}]}f_{\min}(a), f_{\min}(a_n(x_0))\right) \implies x_{n+1}(x_0) \in \left[\min_{a \in [a_{\min},a_{\max}]}f_{\min}(a), f_{\max}(a_n(x_0))\right]
\end{equation*}
and thus the result follows for this one as well. The last case is
\begin{equation*}
    x_n(x_0) \in \left(f_{\max}(a_n(x_0)), \max_{a \in [a_{\min},a_{\max}]}f_{\max}(a)\right] \implies x_{n+1}(x_0) \in \left[f_{\min}(a_n(x_0)), \max_{a \in [a_{\min},a_{\max}]}f_{\max}(a)\right]
\end{equation*}
Clearly the result holds for all cases.
\end{proof}

More interestingly, we  can show that eventually, any subinterval of $(0,1)$ will be absorbed within $\Delta$.
\begin{lemma}[Absorption Property]\label{lemma:absorption-dynamic}
For all $a_{\min} > s_b$ and $[\gamma, \zeta] \subset (0,1)$, there is an $n_0 \geq 0$ so that
\begin{align*}
      \forall n \geq n_0 \quad x_{n}([\gamma, \zeta]) \subseteq \Delta.
\end{align*}
\end{lemma}
\begin{proof}
Let us define the three following sets 
\begin{align*}
    U_0 = \left[\gamma, \min_{a \in [a_{\min},a_{\max}]}f_{\min}(a)\right) \quad 
    \Delta_0 = \Delta \cap [\gamma, \zeta] \quad 
    V_0 = \left(\max_{a \in [a_{\min},a_{\max}]}f_{\max}(a), \zeta\right]
\end{align*}
It suffices to prove the theorem for each of them (if they are non-empty) because we can pick \[n_0 = \max\{n_0(U_0), n_0(\Delta_0), n_0(V_0)\}\] (excluding any empty sets) to satisfy the theorem.
Based on Lemma~\ref{lemma:forward-dynamic}, it is clear that $n_0(\Delta_0) = 0$. Moving on to the case of $U_0$, we know that in $U_0$ we have that $f(x, a, b) > x$. We can define a uniform bound on their difference
\begin{equation*}
    \min_{x \in U_0} \min_{a \in [a_{\min},a_{\max}]} \left[ f(x, a, b) -x \right] = \xi > 0
\end{equation*}
We now know that
\begin{equation*}
    x_1(U_0) \in \left[\gamma + \xi, \max_{a \in [a_{\min},a_{\max}]}f_{\max}(a)\right)
\end{equation*}
If $x_1(U_0) \subset \Delta$ then the result follows trivially for $n_0 = 1$. Otherwise we have that
\begin{equation*}
    x_1(U_0) \in \left[\gamma + \xi, \min_{a \in [a_{\min},a_{\max}]}f_{\min}(a)\right) \subset U_0
\end{equation*}
Applying recursively, either there is a $n_0$ such that $x_{n_0}(U_0) \subset \Delta$ and the theorem holds for $U_0$ or for all $n$ we have that $x_n(U_0)$ stays in a subset of $U_0$
\begin{equation*}
    x_n(U_0) \in \left[\gamma + n\xi, \min_{a \in [a_{\min},a_{\max}]}f_{\min}(a)\right) \subset U_0
\end{equation*}
But this is impossible since there is a $n_0>0$ such that 
\begin{equation*}
 \forall n \geq n_0 : \gamma + n\xi >   \min_{a \in [a_{\min},a_{\max}]}f_{\min}(a) 
\end{equation*}
The case of $V_0$ can be handled symmetrically.
\end{proof}

We note that the key observation for the proof of both the aforementioned lemmas is the fact that at any iteration $x_n(x_0)$ will always come closer to the perpetual absorbing set $\mathcal{F}(a_n(x_0))$ and thus to its closure $\Delta$. 

An immediate consequence of this lemma is the following corollary which will be a key element for the uniform convergence both of our proposed \emph{pseudo-regret} notion and the Ces\`{a}ro mean of the iterations. Analytically, for a given subinterval $[\gamma,\zeta]$ of $(0,1)$, we can always choose some $\delta>0$ such that $\Delta\cup [\gamma,\zeta]\subseteq(\delta, 1 - \delta)$. Again such  choice of $\delta$ is possible because the attraction to $\Delta$ is not eventual phenomenon \textendash\ at each iteration, $\mathrm{dist}(x_n(x_0),\Delta)$ will be either zero or always decreasing . This argumentation is expressed by the following statement:
\begin{corollary}\label{corollary:bound}
Let $a_{\min} > s_b$, then for every $[\gamma, \zeta] \subset (0,1)$, there is an $\delta > 0$ so that
\begin{align*}
      \forall n \geq 0 \quad x_{n}([\gamma, \zeta]) \subseteq (\delta, 1 - \delta).
\end{align*}
\end{corollary}

\subsection{Uniform Convergence of Ces\`{a}ro means \& Learning Rate.}
Having established Corollary~\ref{corollary:bound}, we are now in a position to present our claims regarding the average convergence for any initialization in $(0,1)$. Notably, while the individual iterates of the system, as we shall demonstrate in the ensuing section, exhibit chaotic characteristics, the iterate averages display stabilization over time.

Corollary~\ref{corollary:bound} indicates that for any $x_0 \in (0,1)$, $x_n(x_0)$ will remain bounded away from the endpoints ${0,1}$. We now seek to connect how the boundness away from ${0,1}$ affects the average iterate of a trajectory. By induction on $x_{n+1} = f(x_n, a_n, b)$ we can prove the following 
\begin{equation}
    x_n = \frac{x_0}{x_0 + (1-x_0)\exp\left(\sum_{i=0}^{n-1} a_i(x_i-b)\right)}.\label{eq:MWU-varying}
\end{equation}
Observe that if we choose a $x_0 \in (0,1)$, then $x_n$ is bounded away from ${0,1}$ if and only if $\sum_{i=0}^{n-1} a_i(x_i-b)$ is bounded. Specifically we can prove:
\begin{lemma}\label{ref:average}
Let $a_{\min} > s_b$ and $x_0 \in (0,1)$, then
\begin{equation*}
    \lim_{n \to \infty}\frac{1}{n+1}\sum_{i=0}^n a_i(x_0)(x_i(x_0) - b) = 0.
\end{equation*}
\end{lemma}
\begin{proof}
Observe that $x_0 \in (0,1)$, so by Corollary~\ref{corollary:bound} there is a $\delta > 0$ such that $\delta < x_n < 1-\delta$. Thus we have
\begin{equation*}
    \frac{x_0}{1-\delta} < x_0 + (1-x_0)\exp\left(\sum_{i=0}^{n-1} a_i(x_0) (x_i(x_0)-b)\right) < \frac{x_0}{\delta}
\end{equation*}
so we have that
\begin{equation*}
    \delta^2 < \frac{x_0\delta}{1-\delta} < (1-x_0)\exp\left(\sum_{i=0}^{n-1} a_i(x_0) (x_i(x_0)-b)\right) < \frac{x_0(1-\delta)}{\delta} < \frac{1}{\delta}
\end{equation*}
Clearly we have that
\begin{equation*}
    \delta^2 < \exp\left(\sum_{i=0}^{n-1} a_i(x_0) (x_i(x_0)-b)\right) < \frac{1}{\delta^2}
\end{equation*}
By taking the logarithm and dividing by $n$ we have
\begin{equation*}
    \frac{\ln(\delta^2)}{n} < \frac{1}{n}\sum_{i=0}^{n-1} a_i(x_0) (x_i(x_0)-b) < - \frac{\ln(\delta^2)}{n}
\end{equation*}
By taking the limit, the result follows easily.
\end{proof}

Upon examination of the proof, a salient conclusion is that as $x_0$ moves farther away from the absorbing set $\Delta$ and closer to the endpoints ${0,1}$, the aforementioned convergence rate becomes increasingly slow, and fails in the case that $x_0\in\{0,1\}$, precluding a blanket result for the closed interval $[0,1]$. Conversely, if we restrict our attention to an arbitrary subinterval $[\gamma,\zeta]$, bounded away from the endpoints, we can always exploit the minimum convergence rate of this interval to derive a uniform convergence bound for the learning rate.

\begin{lemma}~\label{ref:uniform}
Let $a_{\min} > s_b$. The sequence of functions $a_n(\cdot)$ is converging uniformly to the constant function $a^\star(\cdot) = g(0)$ in every interval $[\gamma, \zeta] \subset (0,1)$.
\end{lemma}
\begin{proof}
By Corollary~\ref{corollary:bound} $x_0 \in [\gamma, \zeta]$ there is a $\delta>0$ such $\delta < x_n < 1-\delta$. Following the steps of Lemma~\ref{ref:average}
\begin{equation*}
    \frac{\ln(\delta^2)}{n} < \frac{1}{n}\sum_{i=0}^{n-1} a_i(x_0) (x_i(x_0)-b) < - \frac{\ln(\delta^2)}{n}
\end{equation*}
Clearly this implies that for any $n>0$
\begin{equation*}
    a_{n}(x_0) = g\left(\frac{1}{n}\sum_{i=0}^{n-1} a_i(x_0) (x_i(x_0)-b)\right) \in g\left(\left[\frac{\ln(\delta^2)}{n}, - \frac{\ln(\delta^2)}{n}\right]\right) 
\end{equation*}
Since $g$ is continuous, for every $\epsilon>0$ there is a $n_0$ such that $\forall n \geq n_0$ we have that
\begin{equation*}
 \abs{a_n(x_0) - g(0)} \leq \epsilon.   
\end{equation*}
Uniform convergence follows immediately.
\end{proof}

Having established that $a_n(x_0)$ converges to $a^\star=g(0)$ for any $x_0\in(0,1)$, we are able to strengthen Lemma~\ref{ref:average}, demonstrating that its unweighted version of the Ces\`{a}ro mean (i.e. the average iterate in the limit of infinite time) also converges to $b$. 
\begin{lemma}\label{ref:cesaro}
Let $a_{\min} > s_b$ and $x_0 \in (0,1)$, then
\begin{equation*}
    \lim_{n \to \infty}\frac{1}{n+1}\sum_{i=0}^n x_i(x_0) = b.
\end{equation*}
\end{lemma}
\begin{proof}
Let us define $\rho_n = a_n(x_0) - a^\star$. Obviously $\lim_{n \to \infty} \rho_n = 0$. Then applying Lemma~\ref{ref:average} we have
\begin{equation*}
    0 = \lim_{n \to \infty}\frac{1}{n+1}\sum_{i=0}^n a_i(x_0)(x_i(x_0) - b) = \lim_{n \to \infty}\left[\frac{a^\star}{n+1} \sum_{i=0}^n (x_i(x_0) - b)  + \frac{1}{n+1}\sum_{i=0}^n \rho_i(x_i(x_0) - b)\right]
\end{equation*}
Clearly we have that $\lim_{n \to \infty} \rho_n (x_n(x_0) -b) = 0$ since $x_n(x_0)$ is bounded. And thus we obviously have
\begin{equation*}
    0 = \lim_{n \to \infty}\left[\frac{a^\star}{n+1} \sum_{i=0}^n (x_i(x_0) - b) \right]
\end{equation*}
Given that $a^\star$ is positive, the lemma follows.
\end{proof}
It is noteworthy that, despite this being equivalent to the guarantee provided in prior work, e.g. \cite{chotibut2019route}, for the fixed learning rate case, the machinery developed here for the adaptive rate regime has been more complex and involved.

\subsection{Uniform convergence of varying step-size MWU map and volume expansion.}
We now investigate the relationship between $(x_n(x_0),a_n(x_0))$ and the MWU map with a fixed learning rate $f(x_n(x_0),a^\star,b)$ as well as the volume expanding properties of the adaptive learning map. The proofs of the present subsection are deferred to Appendix \ref{app:lemma:dynamic-strong-convergence}. The next lemma will be essential in our proofs of chaos.

\Needspace{28\baselineskip}
\begin{lemma}\label{lemma:dynamic-strong-convergence}
Let $a_{\min}> s_b$. For every $k$, $\epsilon>0$ and $[\gamma, \zeta] \subset (0,1)$
\begin{equation*}
    \exists n_0: \forall n \geq n_0 \ \max_{x_0 \in [\gamma, \zeta]} \abs{x_{n+k}(x_0) - f^k(x_n(x_0), a^\star, b)} \leq \epsilon.
\end{equation*}
\end{lemma}

\begin{figure}[H]
    \centering
    \includegraphics[width=0.6\linewidth]{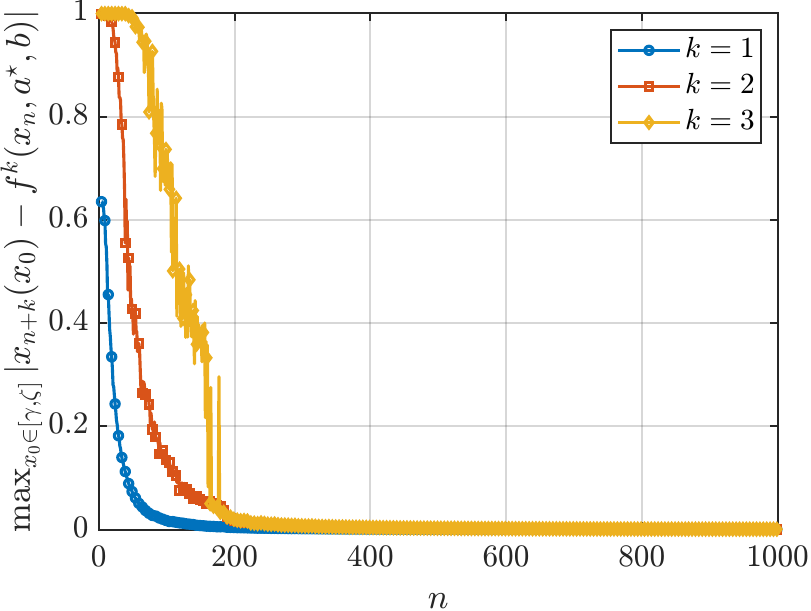}
    \caption{Visualization of the uniform convergence described in Lemma \ref{lemma:dynamic-strong-convergence} for the particular system described in \eqref{sys_example} of Section \ref{sec:sec 6}.}
    \label{fig:dynamic-strong-convergence}
\end{figure}

We can intuitively think of this result as a form of uniform convergence result for the sequence of MWU maps with varying learning rates to the fixed learning rate map. However, as we stress in Section \ref{sec:sec 6}, this property should not be confused with an orbit of \eqref{eq:old-update} and the MWU system with $a^\star$, with the same $x_0$ converging to the same attractor. Fig. \ref{fig:dynamic-strong-convergence} illustrates Lemma \ref{lemma:dynamic-strong-convergence}, for the particular system described in \eqref{sys_example} of Section \ref{sec:sec 6}. 

We now prove two volume expansion lemmas for the dynamic learning rate case, namely that neighborhoods of $b$ within $\mathcal{D}(a_{\min})$ eventually expand to at least $\mathcal{F}(a_{\min})$. The intuition  behind this result is based on the monotonicity property of $\mathcal{F}(a)$. Specifically, for high enough learning rates $a_1\leq a_2$, we have that $\mathcal{F}(a_1)\subseteq \mathcal{F}(a_2)$. Fig. \ref{vol_exp} provides an illustration of the following lemma.

\Needspace{28\baselineskip}
\begin{lemma}\label{lemma:dynamic-learning:volume-expansion-basic}
For every $b \in (0,1)$, there is a $z_b$ such that for all $a_{\min} > z_b$ it holds that for every $[\gamma, \zeta]$  such that $\{ b\} \subset [\gamma, \zeta] \subseteq \mathcal{D}(a_{\min})=[x_{\max}(a_{\min}), x_{\min}(a_{\min})]$
\begin{equation*}
     \exists n_0: \ \forall n \geq n_0 \ \  x_{n}([\gamma, \zeta]) \supseteq %[f_{\min}(a_{\min}), f_{\max}(a_{\min})]=
     \mathcal{F}(a_{\min}).
\end{equation*}
\end{lemma}

\begin{figure}[H]
    \centering
    \includegraphics[width=0.6\linewidth]{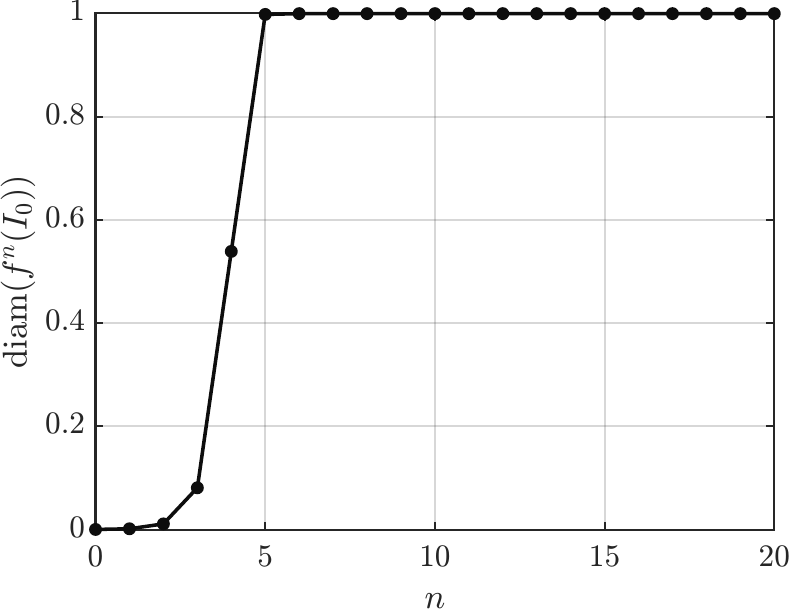}
    \caption{Illustration of the volume expansion as described in Lemma \ref{lemma:dynamic-learning:volume-expansion-basic}}
    \label{vol_exp}
\end{figure}

It is worth mentioning that the series of Lemmas \ref{lemma:forward-dynamic}, \ref{lemma:absorption-dynamic}, \ref{ref:average} \ref{ref:uniform} and \ref{ref:cesaro},\ref{lemma:dynamic-learning:volume-expansion-basic} of this section \textendash pertaining to forward invariance, absorption, volume expansion, and convergence of Ces\`{a}ro means \textendash hold actually regardless of the choice of the update rule and are of independent interest. For the case of the fixed learning rate, Theorem~\ref{thm:volume-expansion} makes a more refined prediction compared to Lemma~\ref{lemma:dynamic-learning:volume-expansion-basic} as the former provides an equality. In the dynamic learning rate case there is a gap between the eventual image upper bound, $\Delta$ and the volume expansion lower bound $\mathcal{F}(a_{\min})$. In order to close this gap we will make use of the fact that the learning rate converges to $g(0)$.

Building upon the full range of the developed machinery of this section, we will strengthen the above volume expansion result  by showing that $ a_{\min}$ could be actually substituted by any $\tilde{a}=a^\star-\epsilon$ for any sufficiently small $\epsilon > 0$ such that  $a^\star-\epsilon >a_{\min} > z_b$. Our first observation is that thanks to the uniform convergence of $a_n \to a^\star$ in $[\gamma, \zeta] \subset (0,1)$ there exists a $n^{\dag}$ such that 
% \[
% \forall n \geq n^{\dag} \quad a_{n}([\gamma, \delta]) \subseteq (a^\star-\epsilon, a^\star+\epsilon)
% \]
% which implies that  there exists a $n^{\dag}$ such that
\begin{equation*}
    \forall n > n^{\dag} \quad \forall x \in [\gamma, \zeta] \quad a_n(x) > a^\star-\epsilon = \tilde{a}
\end{equation*}
This strengthens the previous volume expansion result to at least the set $\mathcal{F}(\tilde{a})$ for any $\tilde{a}$ close to $a^\star$.
\begin{theorem}\label{theorem:dynamic-volume-expansion}
For $a_{\min} > z_b$ and for any sufficiently small $\epsilon > 0$ such that  $a^\star-\epsilon >a_{\min} > z_b$ we have that for all $[\gamma, \zeta]$ such that 
$\{ b\} \subset [\gamma, \zeta] \subset (0,1)$, it holds that
\begin{equation*}
      \exists n_0: \ \forall n \geq n_0 \quad x_{n}([\gamma, \zeta]) \supseteq \mathcal{F}(a^\star-\epsilon).
\end{equation*}
\end{theorem}

In a completely similar fashion we can show that the absorbing set $\Delta$ can also be refined. For a given interval  $[\gamma, \zeta] \subset (0,1)$, as the learning rates converge to $a^\star$, the trajectories will tend to be absorbed by $\mathcal{F}(a^\star)$. Thus the dynamic learning rate volume expansion behavior matches fixed learning rate case in the long run.

\section{Turbulent Sets and Chaos in the Adaptive MWU map }
\label{sec:turbulence}
The roadmap of this section is our  construction of the turbulent sets and
their connection with the symbolic dynamics in order to prove Li--Yorke chaos in our non-autonomous dynamical system. It should be noticed that our novel approach is a major departure from the standard techniques of period-3 orbit arguments that have been extensively used in the case of fixed learning rates. We start with some useful definition for our reduction.

\begin{definition}[\cite{shi2006study}]
A continuous map $f:I\to I$ is called \emph{turbulent} if there exist compact subintervals $J,K$ with at most one common point such that $J\cup K\subseteq f(J)\cap f(K)$
Additionally, the map is called \emph{strictly turbulent} if the subintervals can be chosen to be disjoint. Finally, the corresponding sets $J,K$ are called turbulent sets.
\end{definition}

Delving into the proof of \cite{liyorke}, it becomes clear that the chaotic map $f(x,a,b)$ has a period-3 orbit that exists within the interior of $\mathcal{F}(a)$. We take advantage of this property to demonstrate that $f^2(\cdot,a,b)$ is in fact a strictly turbulent map.

% \begin{lemma}\label{lemma:turbulence}
% For every $b \in (0,1)$, $f^2(\cdot,a,b)$  is in fact a strict turbulent map for sufficiently large $a$, i.e., there exist closed and disjoint intervals $K_a$ and $J_a$ in the interior of $\mathcal{F}(a)$ such that
% $f^2(K_a, a, b)$ and $f^2(J_a, a, b)$ are neighborhoods of $K_a \cup J_a$.
% \end{lemma}

\Needspace{27\baselineskip}
\begin{lemma}\label{lemma:turbulence}
For every $b \in (0,1)$, there is a $u_b$ such that for all $a> u_b$,  there exist closed and disjoint intervals $K_a$ and $J_a$ in the interior of $\mathcal{F}(a)$ such that
$f^2(K_a, a, b)$ and $f^2(J_a, a, b)$ are neighborhoods of $K_a \cup J_a$.
\end{lemma}

\begin{figure}[H]
    \centering
    \includegraphics[width=0.6\linewidth]{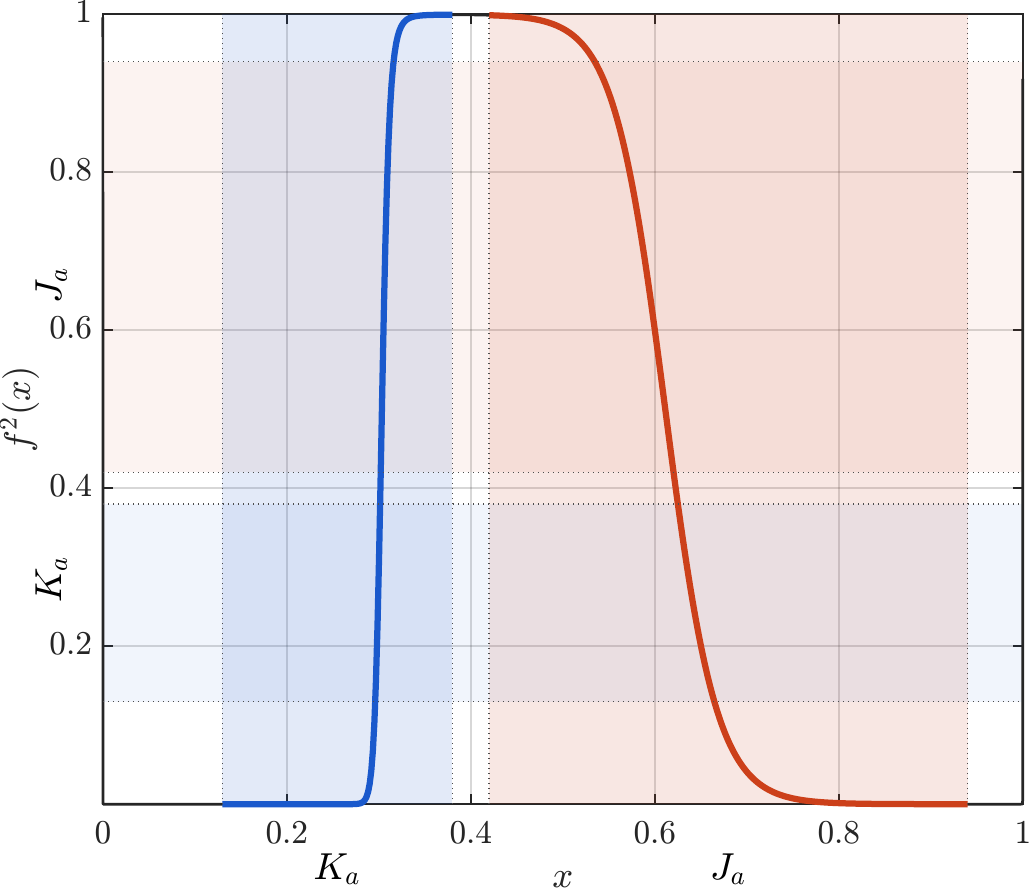}
    \caption{Illustration of the turbulent sets described in Lemma \ref{lemma:turbulence}.}
    \label{fig:turbulence}
\end{figure}

\begin{proof}
As shown in Theorem \ref{theorem_turb_1} of Appendix \ref{app:turbulence}, if $b \in (0,1) \setminus \{\frac{1}{2}\}$ and for all $a>u_b$, the map $f(x,a,b)$ has a period-3 orbit in the interior of $\mathcal{F}(a)=[f_{\text{min}}(a), f_{\text{max}}(a)]$. This period-3 orbit has the form $f^2(x, a, b)< x< f(x, a, b) $ or its mirror image (See Theorem 1 in \cite{liyorke}). Without loss of generality, assume it has the form above. Then we can choose,
\begin{enumerate}
\item $d$  between $x$ and $f(x, a, b)$, so that $f(d,a, b)=x$ and hence $d<f^2(d,a, b)$. 
\item $z$ between $f(x)$ and $x$, so close that $f^2(z, a, b) >d$.
\item $q$ between $z$ and $x$, so close to $x$ that $f^2(q, a, b) < z$.
\item $c$ between $x$ and $d$, so close to $x$ that $f^2(c, a, b) < z$.
\end{enumerate} 
 Then, $J_a = [z,q]$ and $K_a = [c,d]$ are disjoint and
    \[  [z,d] \subset [f^2(q, a, b),f^2(z, a, b)] \subset  f^2(J_a),
        [z,d] \subset [f^2(c, a, b),f^2(d, a, b)] \subset  f^2(K_a)\]
Observe that $f^2(K_a, a, b)$ and $f^2(J_a, a, b)$ are supersets of $K_a$ and $J_a$ that share no endpoints with $K_a$ and $J_a$ so they are neighborhoods of $K_a \cup J_a$. Also since $x$, $f(x, a, b)$ and $f^2(x, a, b)$ are in the interior of the interval $[f_{\min}(a), f_{\max}(a)]$, we have that $J_a$ and $K_a$ have this property as well.
\end{proof}

Fig. \ref{fig:turbulence} illustrates two sets as described in Lemma \ref{lemma:turbulence}, following the construction of the proof. We now examine the properties of a fixed learning rate MWU map and its ability to create exponential decaying volume (length) turbulent sets. Our analysis centers around the key observation that since $f^2(K_a)$ covers $K_a\cup J_a$, then there exist necessarily, by continuity of $f^2$, at least two distinct subintervals, $Z_1,Z_2$, within $K_a$ whose $f^2$-image is precisely $K_a$ and $J_a$, respectively. By repeatedly applying this principle, we demonstrate that the $f^4$-image of these subintervals cover again $K_a\cup J_a$, and through induction, we can extend this property to higher compositions of $f$.

\Needspace{28\baselineskip}
\begin{lemma}\label{lemma:turbulence-advanced}
For every $a > u_b$, there exist closed intervals $V_a^k \subseteq K_a$ and $U_a^k \subseteq J_a$ such that
\begin{align*}
    \lim_{k \to \infty} \text{diam}(V_a^k) = \lim_{k \to \infty} \text{diam}(U_a^k) = 0
\end{align*}
and for every $k \geq 0$ it holds that $f^{2k + 2}(V_a^k, a, b)$ and $f^{2k + 2}(U_a^k, a, b)$ are neighborhoods of $K_a \cup J_a$.
\end{lemma}

\begin{figure}[H]
    \centering

    % -------------------- (a) --------------------
    \begin{minipage}[c]{0.47\linewidth}
        \centering
        \resizebox{\linewidth}{!}{\tikzset{every picture/.style={line width=0.75pt}} %set default line width to 0.75pt        

\begin{tikzpicture}[x=0.75pt,y=0.75pt,yscale=-1,xscale=1]
%uncomment if require: \path (0,445); %set diagram left start at 0, and has height of 445

%Straight Lines [id:da5394879690868122] 
\draw    (400.41,120.19) -- (220.22,120.63) ;
%Shape: Rectangle [id:dp22620449667588238] 
\draw   (353.11,80.07) .. controls (368.03,80.03) and (380.13,92.09) .. (380.16,107) -- (380.21,133.48) .. controls (380.23,148.39) and (368.17,160.51) .. (353.25,160.54) -- (267.52,160.75) .. controls (252.61,160.79) and (240.5,148.73) .. (240.47,133.82) -- (240.43,107.34) .. controls (240.4,92.43) and (252.47,80.31) .. (267.38,80.28) -- cycle ;
%Shape: Rectangle [id:dp7228055996711912] 
\draw  [fill={rgb, 255:red, 245; green, 166; blue, 35 }  ,fill opacity=0.38 ] (350.15,93.6) .. controls (362.46,93.57) and (372.47,104.36) .. (372.49,117.7) .. controls (372.51,131.03) and (362.55,141.87) .. (350.23,141.9) .. controls (337.92,141.93) and (327.92,131.14) .. (327.89,117.8) .. controls (327.87,104.47) and (337.83,93.63) .. (350.15,93.6) -- cycle ;
%Shape: Rectangle [id:dp7805304186490449] 
\draw  [fill={rgb, 255:red, 80; green, 227; blue, 194 }  ,fill opacity=0.54 ] (282.54,92.74) .. controls (295.65,92.71) and (306.31,104.41) .. (306.33,118.88) .. controls (306.36,133.34) and (295.75,145.09) .. (282.63,145.12) .. controls (269.51,145.16) and (258.86,133.46) .. (258.83,118.99) .. controls (258.81,104.53) and (269.42,92.78) .. (282.54,92.74) -- cycle ;
%Shape: Rectangle [id:dp26722576184194435] 
\draw  [fill={rgb, 255:red, 245; green, 166; blue, 35 }  ,fill opacity=0.23 ] (296.78,115.49) -- (296.8,125.91) -- (282.26,125.94) -- (282.25,115.53) -- cycle ;
%Shape: Rectangle [id:dp7114197740947467] 
\draw  [fill={rgb, 255:red, 80; green, 227; blue, 194 }  ,fill opacity=0.48 ] (276.3,115.65) -- (276.32,125.96) -- (264.42,125.99) -- (264.4,115.68) -- cycle ;
%Shape: Rectangle [id:dp28713158732908206] 
\draw  [fill={rgb, 255:red, 245; green, 166; blue, 35 }  ,fill opacity=0.45 ] (365.91,115.33) -- (365.93,125.05) -- (355.41,125.08) -- (355.39,115.35) -- cycle ;
%Shape: Rectangle [id:dp7949640485877618] 
\draw  [fill={rgb, 255:red, 80; green, 227; blue, 194 }  ,fill opacity=0.26 ] (347.73,115.48) -- (347.75,124.52) -- (333.55,124.55) -- (333.54,115.52) -- cycle ;
%Curve Lines [id:da6223172899952368] 
\draw    (323.42,214.36) .. controls (267.74,180.77) and (328.08,172.6) .. (341.46,124.52) ;
%Curve Lines [id:da85850722726211] 
\draw    (389.38,214.36) .. controls (428.78,175.68) and (380.81,182.8) .. (357.81,125.38) ;
%Straight Lines [id:da23292594611282658] 
\draw    (347.41,141.69) -- (347.41,198.11) ;
\draw [shift={(347.41,200.11)}, rotate = 270] [color={rgb, 255:red, 0; green, 0; blue, 0 }  ][line width=0.75]    (10.93,-3.29) .. controls (6.95,-1.4) and (3.31,-0.3) .. (0,0) .. controls (3.31,0.3) and (6.95,1.4) .. (10.93,3.29)   ;
%Shape: Rectangle [id:dp7163988727896183] 
\draw  [fill={rgb, 255:red, 245; green, 166; blue, 35 }  ,fill opacity=0.38 ] (380.38,214.42) .. controls (392.89,214.39) and (403.05,219.15) .. (403.06,225.05) .. controls (403.07,230.95) and (392.94,235.75) .. (380.43,235.78) .. controls (367.91,235.8) and (357.76,231.04) .. (357.74,225.15) .. controls (357.73,219.25) and (367.86,214.45) .. (380.38,214.42) -- cycle ;
%Shape: Rectangle [id:dp0569652672292289] 
\draw  [fill={rgb, 255:red, 80; green, 227; blue, 194 }  ,fill opacity=0.54 ] (320.97,214.87) .. controls (333.49,214.85) and (343.64,219.61) .. (343.66,225.51) .. controls (343.67,231.4) and (333.54,236.21) .. (321.02,236.23) .. controls (308.51,236.26) and (298.35,231.5) .. (298.34,225.6) .. controls (298.33,219.7) and (308.46,214.9) .. (320.97,214.87) -- cycle ;
%Shape: Rectangle [id:dp05967747812293944] 
\draw   (383.74,199.91) .. controls (398.65,199.88) and (410.77,210.79) .. (410.8,224.29) .. controls (410.83,237.78) and (398.76,248.74) .. (383.85,248.77) -- (317.87,248.91) .. controls (302.96,248.94) and (290.85,238.03) .. (290.82,224.54) .. controls (290.79,211.05) and (302.85,200.08) .. (317.76,200.05) -- cycle ;
%Straight Lines [id:da8214513088116734] 
\draw    (389.24,236.75) -- (413.36,300.02) ;
\draw [shift={(414.08,301.89)}, rotate = 249.13] [color={rgb, 255:red, 0; green, 0; blue, 0 }  ][line width=0.75]    (10.93,-3.29) .. controls (6.95,-1.4) and (3.31,-0.3) .. (0,0) .. controls (3.31,0.3) and (6.95,1.4) .. (10.93,3.29)   ;
%Straight Lines [id:da26464971567954265] 
\draw    (311.14,234.75) -- (272.74,301.21) ;
\draw [shift={(271.74,302.95)}, rotate = 300.02] [color={rgb, 255:red, 0; green, 0; blue, 0 }  ][line width=0.75]    (10.93,-3.29) .. controls (6.95,-1.4) and (3.31,-0.3) .. (0,0) .. controls (3.31,0.3) and (6.95,1.4) .. (10.93,3.29)   ;
%Shape: Rectangle [id:dp42175641186085866] 
\draw  [fill={rgb, 255:red, 245; green, 166; blue, 35 }  ,fill opacity=0.38 ] (297.06,308.71) .. controls (305.88,308.69) and (313.04,312.32) .. (313.05,316.82) .. controls (313.06,321.32) and (305.91,324.98) .. (297.09,325) .. controls (288.27,325.02) and (281.11,321.39) .. (281.1,316.89) .. controls (281.09,312.4) and (288.23,308.73) .. (297.06,308.71) -- cycle ;
%Shape: Rectangle [id:dp032664278955003456] 
\draw  [fill={rgb, 255:red, 80; green, 227; blue, 194 }  ,fill opacity=0.54 ] (247.14,308.82) .. controls (255.96,308.8) and (263.12,312.43) .. (263.13,316.93) .. controls (263.14,321.43) and (256,325.09) .. (247.17,325.11) .. controls (238.35,325.13) and (231.19,321.5) .. (231.18,317) .. controls (231.17,312.5) and (238.32,308.84) .. (247.14,308.82) -- cycle ;
%Shape: Rectangle [id:dp8818136761661548] 
\draw   (295.79,301.88) .. controls (310.7,301.84) and (322.8,308.2) .. (322.82,316.07) .. controls (322.84,323.94) and (310.76,330.34) .. (295.85,330.38) -- (252.44,330.47) .. controls (237.53,330.51) and (225.43,324.16) .. (225.41,316.29) .. controls (225.4,308.42) and (237.47,302.01) .. (252.38,301.98) -- cycle ;
%Shape: Rectangle [id:dp6116826652319777] 
\draw  [fill={rgb, 255:red, 245; green, 166; blue, 35 }  ,fill opacity=0.38 ] (437.68,307.66) .. controls (446.5,307.64) and (453.66,311.27) .. (453.67,315.77) .. controls (453.68,320.27) and (446.54,323.93) .. (437.72,323.95) .. controls (428.89,323.97) and (421.74,320.34) .. (421.73,315.84) .. controls (421.72,311.34) and (428.86,307.68) .. (437.68,307.66) -- cycle ;
%Shape: Rectangle [id:dp6206250049588424] 
\draw  [fill={rgb, 255:red, 80; green, 227; blue, 194 }  ,fill opacity=0.54 ] (387.77,307.77) .. controls (396.59,307.75) and (403.75,311.38) .. (403.76,315.87) .. controls (403.77,320.37) and (396.62,324.04) .. (387.8,324.06) .. controls (378.98,324.08) and (371.82,320.44) .. (371.81,315.95) .. controls (371.8,311.45) and (378.94,307.79) .. (387.77,307.77) -- cycle ;
%Shape: Rectangle [id:dp8474007656019541] 
\draw   (436.42,300.82) .. controls (451.33,300.79) and (463.43,307.14) .. (463.45,315.01) .. controls (463.46,322.88) and (451.39,329.29) .. (436.48,329.32) -- (393.07,329.42) .. controls (378.16,329.45) and (366.06,323.1) .. (366.04,315.23) .. controls (366.02,307.36) and (378.1,300.96) .. (393.01,300.92) -- cycle ;

% Text Node
\draw (380.4,225.1) node  [font=\footnotesize]  {$K_{a}$};
% Text Node
\draw (321,225.55) node  [font=\footnotesize]  {$J_{a}$};
% Text Node
\draw (297.05,258.69) node [anchor=north west][inner sep=0.75pt]  [font=\footnotesize]  {$f^{2}\left( V_{a}^{0} =K_{a} ,a,b\right)$};
% Text Node
\draw (370.6,334.16) node [anchor=north west][inner sep=0.75pt]  [font=\footnotesize]  {$f^{4}\left( V_{a}^{1} =Z_{2} ,a,b\right)$};
% Text Node
\draw (241.45,335.3) node [anchor=north west][inner sep=0.75pt]  [font=\footnotesize]  {$f^{4}( Z_{1} ,a,b)$};
% Text Node
\draw (267.29,191.53) node [anchor=north west][inner sep=0.75pt]  [font=\LARGE,rotate=-90.31]  {$\ddots $};
% Text Node
\draw (240.32,241.11) node [anchor=north west][inner sep=0.75pt]  [font=\LARGE,rotate=-89.16]  {$\ddots $};
% Text Node
\draw (218.72,156.67) node [anchor=north west][inner sep=0.75pt]  [font=\small]  {$\mathcal{F}( a)$};
% Text Node
\draw (378.15,125.92) node [anchor=north west][inner sep=0.75pt]  [font=\small]  {$f_{\max}( a)$};
% Text Node
\draw (191.83,125.31) node [anchor=north west][inner sep=0.75pt]  [font=\small]  {$f_{\min}( a)$};
% Text Node
\draw (309.25,122.11) node [anchor=north west][inner sep=0.75pt]  [rotate=-358.71]  {$K_{a}$};
% Text Node
\draw (245.21,122.1) node [anchor=north west][inner sep=0.75pt]  [rotate=-358.71]  {$J_{a}$};
% Text Node
\draw (263.22,99.87) node [anchor=north west][inner sep=0.75pt]  [font=\footnotesize]  {$\Xi _{1}$};
% Text Node
\draw (281.84,99.87) node [anchor=north west][inner sep=0.75pt]  [font=\footnotesize]  {$\Xi _{2}$};
% Text Node
\draw (333.21,99.87) node [anchor=north west][inner sep=0.75pt]  [font=\footnotesize]  {$Z_{1}$};
% Text Node
\draw (350.71,99.87) node [anchor=north west][inner sep=0.75pt]  [font=\footnotesize]  {$Z_{2}$};
% Text Node
\draw (290.25,311.44) node [anchor=north west][inner sep=0.75pt]  [font=\scriptsize]  {$K_{a}$};
% Text Node
\draw (242.07,311.63) node [anchor=north west][inner sep=0.75pt]  [font=\scriptsize]  {$J_{a}$};
% Text Node
\draw (431.74,311.4) node [anchor=north west][inner sep=0.75pt]  [font=\scriptsize]  {$K_{a}$};
% Text Node
\draw (383.12,311.08) node [anchor=north west][inner sep=0.75pt]  [font=\scriptsize]  {$J_{a}$};

\end{tikzpicture}}

        \vspace{2mm}
        {\small (a)}
    \end{minipage}
    \hfill
    % -------------------- (b) --------------------
    \begin{minipage}[c]{0.5\linewidth}
        \centering
        \includegraphics[width=\linewidth]{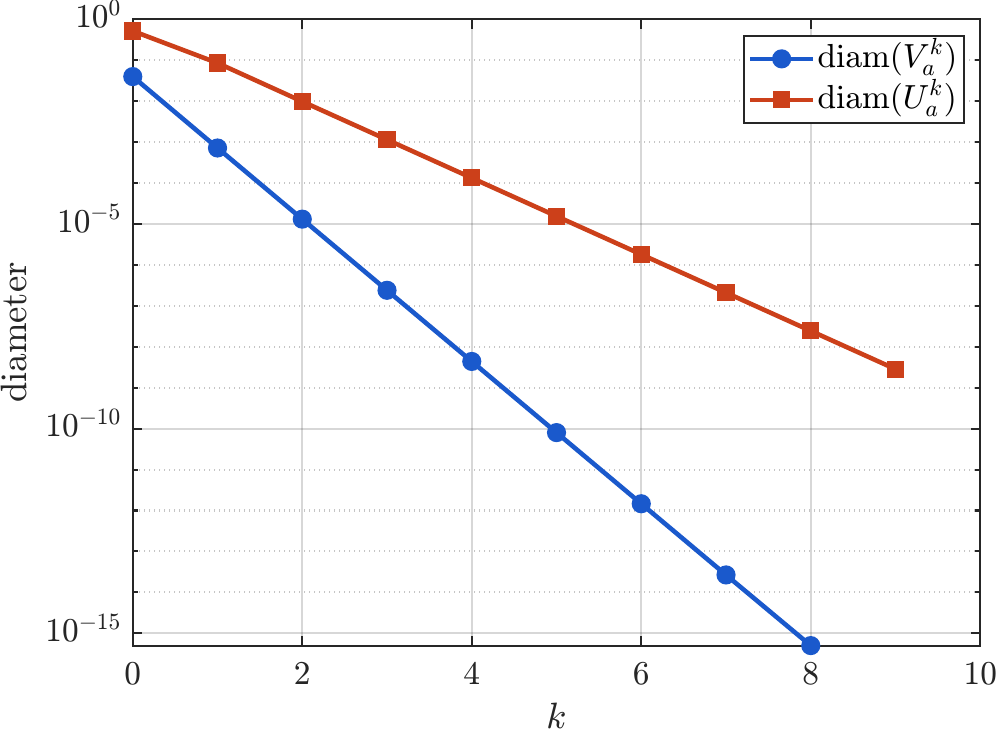}

        \vspace{2mm}
        {\small (b)}
    \end{minipage}

    \caption{Illustration of the construction described in Lemma~\ref{lemma:turbulence-advanced}.}
    \label{fig:turbulence-advanced}
\end{figure}

\begin{proof}
We will prove the lemma by induction. Choosing $V_a^0 = K_a$ we know by Lemma~\ref{lemma:turbulence} that $f^2(V_a^0, a, b)$ is a neighborhood of $K_a \cup J_a$. Now let us assume that we have a $V_a^k \subseteq K_a$ such that $f^{2k + 2}(V_a^k, a, b)$ is a neighborhood of $K_a \cup J_a$. Since $K_a$ and $J_a$ are disjoint closed intervals , there are two disjoint closed intervals $Z_1$ and $Z_2$ of $V_a^k$ such that $f^{2k + 2}(Z_1, a, b)$ is a neighborhood of $K_a$ and $f^{2k + 2}(Z_2, a, b)$ is a neighborhood of $J_a$. As a result $f^{2k+4}(Z_1, a, b)$ and $f^{2k+4}(Z_2, a, b)$ are neighborhoods of $K_a \cup J_a$. Given that $Z_1$ and $Z_2$ are disjoint intervals of $V_a^k$ it must be the case that
\begin{equation*}
    \min\{\textrm{diam}(Z_1), \textrm{diam}(Z_2)\} \leq \textrm{diam}(V_a^k)/2.
\end{equation*}
We pick the interval with the smallest diameter as $V_a^{k+1}$ and thus 
\begin{equation*}
    0 \leq \textrm{diam}(V_a^k) \leq 2^{-k} \textrm{diam}(K_a) \implies \lim_{k \to \infty} \textrm{diam}(V_a^k) = 0.
\end{equation*}
Choosing $U_a^0 = J_a$ we can follow the same arguments for the rest of the $U_a^k$.
\end{proof}

% \begin{lemma}\label{lemma:turbulence-advanced}
% For every $b \in (0,1)$, and sufficiently large step-size,\\ there exist closed intervals $V_a^k \subseteq K_a$ and $U_a^k \subseteq J_a$ such that
% \begin{align*}
%     \lim_{k \to \infty} \textrm{diam}(V_a^k) = \lim_{k \to \infty} \textrm{diam}(U_a^k) = 0
% \end{align*}
% and for every $k \geq 0$ it holds that $f^{2k + 2}(V_a^k, a, b)$ and $f^{2k + 2}(U_a^k, a, b)$ are neighborhoods of $K_a \cup J_a$.
% \end{lemma}

The following lemma plays  an essential role in our symbolic dynamic proof of the Li--Yorke chaotic behavior. It allows us to construct a scrambled set of initial conditions through a set of abstract symbolic orbits. An illustration based on the construction described in the proof is shown in Fig. \ref{fig:turbulence-advanced}.

On a technical level, we utilize a range of machinery developed in this paper to prove this lemma. Specifically, we use Theorem~\ref{theorem:dynamic-volume-expansion} to firstly describe turbulent sets $J_{a^\star}\cup K_{a^\star}$ that lie within the interior of $\mathcal{F}(a^\star)$ and, for small enough $\epsilon$, within $\mathcal{F}(a^\star-\epsilon)$, and secondly to ensure that the non-autonomous dynamical system covers $J_{a^\star}\cup K_{a^\star}$.
Furthermore, to extend the implication of Lemma~\ref{lemma:turbulence-advanced} for the turbulent map $f^2(\cdot,a^\star,b)$ to the non-autonomous system $x_n$ map for the sets $J_{a^\star},K_{a^\star}$, we employ the uniform convergence guarantee provided by Lemma~\ref{lemma:dynamic-strong-convergence}, which controls the discrepancy between $x_{n+2}(S)$ and $f^2(x_n(S),a^\star,b)$ for any subinterval $S$ of $J_{a^\star},K_{a^\star}$. 

Our approach to demonstrate the Li--Yorke chaotic behavior through symbolic dynamics is summarized in the following high-level steps: 
\begin{itemize}
\item 
Assume that we can construct an uncountable set $\mathcal{S}$ of scrambled infinite length binary sequences, i.e., 
 for every pair of sequences $\sigma,\tau \in \mathcal{S}$ there exist $(i)$ an infinite length subsequence $(k_i)_{i\in \mathbb{N}}$ where the two sequences differ, i.e., $\sigma_{k_i}\neq\tau_{k_i}$ and $(ii)$ an infinite length subsequence  $(\ell_i)_{i\in \mathbb{N}}$ where the two sequences are equal, i.e., $\sigma_{\ell_i}=\tau_{\ell_i}$.
 \item For each element of $\sigma\in \mathcal{S}$ we construct a sequence of sets $\mathcal{A}_\sigma=(A_i)_{i\in \mathbb{N}}$ as follows: If the $k$-th place element is $0$ we use $V_{a^\star}^k$, whereas if it is $1$ we pick $U_{a^\star}^k$. We now apply Lemma~\ref{lemma:tracking} for each of the sequence of sets to get a corresponding initialization $x_0^{\sigma}$. We call this set of initializations $Q$.
 \item Since every pair of strings $\sigma,\tau$ is scrambled, we know that we can construct two infinite subsequences $\mu, \; \nu$ such that $\{x_{n_i}(x_0^{\sigma}) : i\in\mu\}$, $\{x_{n_i}(x_0^{\tau}) : i\in\mu\}$ belong to the same turbulent sets and $\{x_{n_i}(x_0^{\sigma}) : i\in\nu\}$, $\{x_{n_i}(x_0^{\tau}) : i\in\nu\}$ belong to the disjoint turbulent ones. Therefore, we can show that for every pair $\sigma,\tau\in Q$: $
 (i) \liminf\|\{x_{n}(x_0^{\sigma})-x_{n}(x_0^{\tau})\}\|=0
 \quad (ii) \limsup\|\{x_{n}(x_0^{\sigma})-x_{n}(x_0^{\tau})\}\|>0
$
\end{itemize}

\Needspace{23\baselineskip}
\begin{lemma}[Tracking Lemma]\label{lemma:tracking}
If $b \in (0, 1) \setminus \{\frac{1}{2} \}$, there exists a $d_b$ such that if $a_{\min} > d_b$, we can construct an increasing sequence $n_i$ with the following properties. For every sequence of intervals $A_i$ with $A_i = V_{a^\star}^i$ or $A_i = U_{a^\star}^i$, there exists a $x_0 \in [0,1]$ such that for all $i \geq 0$ it holds that $x_{n_i}(x_0) \in A_i$. 
\end{lemma}

\begin{figure}[H]
    \centering
  \begin{subfigure}[b]{0.43\linewidth}
    \centering
    \includegraphics[width=\linewidth]{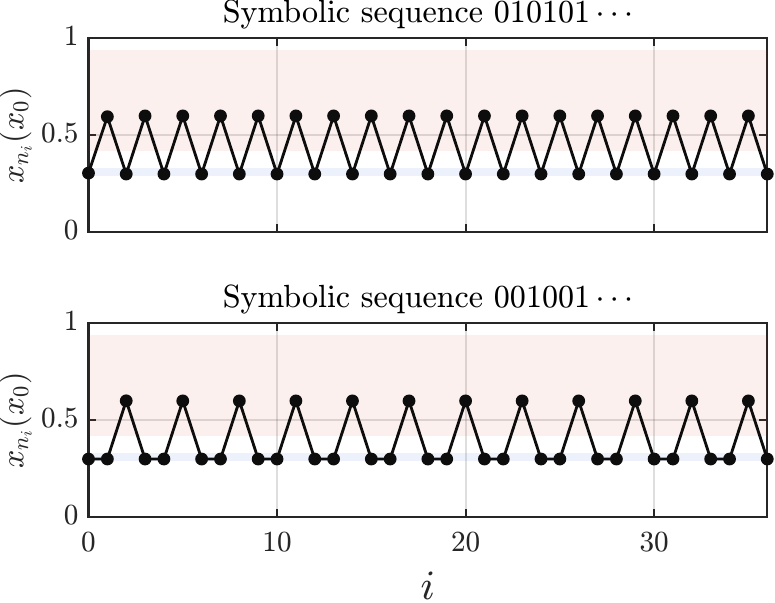}
    \caption{}
  \end{subfigure}
  \hfill
  \begin{subfigure}[b]{0.43\linewidth}
    \centering
    \includegraphics[width=\linewidth]{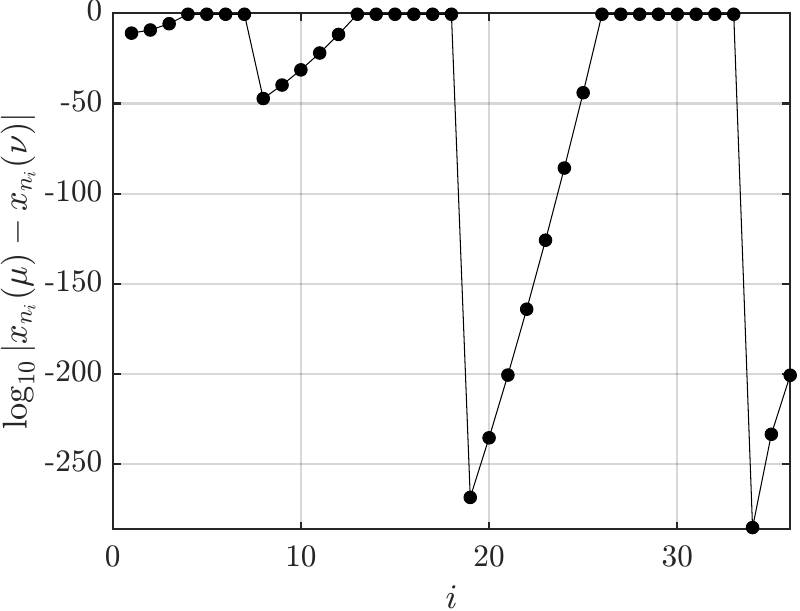}
    \caption{}
  \end{subfigure}
  \captionsetup{justification=justified,singlelinecheck=false}
  \caption{(a) Illustration of the tracking construction from Lemma~\ref{lemma:tracking}: two specific subsequences satisfy $x_{n_i} \in A_i$. (b)~The trajectories repeatedly approach one another and separate again, with \mbox{$\liminf |x_{n_i}(\mu)-x_{n_i}(\nu)|=0$ but $\limsup |x_{n_i}(\mu)-x_{n_i}(\nu)|>0$}. These limits illustrate the scrambled behavior used in the proof.}
\label{fig: lem12}
%\vspace*{12pt}
\end{figure}

\begin{proof}
By continuity of $f_{\min}$ and $f_{\max}$ we know that
\begin{equation*}
    \lim_{\epsilon \to 0} f_{\min}(a^\star - \epsilon) = f_{\min}(a^\star) \quad \lim_{\epsilon \to 0} f_{\max}(a^\star - \epsilon) = f_{\max}(a^\star) 
\end{equation*}
By Lemma~\ref{lemma:turbulence} we know that $K_{a^\star}$ and $J_{a^\star}$ are in the interior of $[f_{\min}(a^\star), f_{\max}(a^\star)]$. As a result there is a sufficiently small $\epsilon > 0$ such that both of the following properties hold
\begin{gather*}
    a^\star > a^\star - \epsilon > a_{\min} \\
    [f_{\min}(a^\star - \epsilon), f_{\max}(a^\star - \epsilon)] \supseteq K_{a^\star} \cup J_{a^\star} 
\end{gather*}
Let us pick any interval $[\gamma, \zeta]$ such that $\{b\} \subset [\gamma, \zeta] \subset (0,1)$. By Theorem~\ref{app:theorem:dynamic-volume-expansion} and for the aforementioned $\epsilon$, we know that there is a $n^*$ such that
\begin{equation*}
    \forall n \geq n^* : \quad x_n([\gamma, \zeta]) \supseteq  [f_{\min}(a^\star - \epsilon), f_{\max}(a^\star - \epsilon)] \supseteq K_{a^\star} \cup J_{a^\star}
\end{equation*}
We also know that for all $i \geq 0$, $f^{2i +2}(V_{a^\star}^i, a^\star, b)$ and $f^{2i +2}(U_{a^\star}^i, a^\star, b)$ are neighborhoods of $K_{a^\star} \cup J_{a^\star}$. Thus for each $i \geq 0$ there must be an $\epsilon_i > 0$ such that $f^{2i +2}(V_{a^\star}^i, a^\star, b)$ and $f^{2i +2}(U_{a^\star}^i, a^\star, b)$ are $\epsilon_i$ neighborhoods of $K_{a^\star} \cup J_{a^\star}$. By Lemma~\ref{lemma:dynamic-strong-convergence}, there is a sequence of $m_i$ such that for all $i \geq 0$
\begin{equation*}
    \forall n \geq m_i \quad \max_{x_0 \in [\gamma, \zeta]} \abs{x_{n+2i+2}(x_0) - f^{2i+2}(x_n(x_0), a^\star, b)} \leq \epsilon_i
\end{equation*}
We are now ready to construct the sequence $n_i$. We choose $n_0 = \max\{m_0, n^*\}$. For $n_i$ with $i \geq 1$, we choose the minimum number with the following properties
\begin{gather*}
    (n_{i} - n_{i-1} -2i) \mod 2 = 0 \\
    n_{i} \geq n_{i-1} +2i \quad  n_{i} \geq m_i
\end{gather*}
It now remains to construct the required $x_0 \in [0,1]$ for each potential sequence of $A_i$. Since $x_{n_0}([\gamma, \zeta])$ is a superset of the disjoint closed intervals $K_{a^\star}=V_{a^\star}^0$ and $J_{a^\star}=U_{a^\star}^0$ we know that there are closed minimal intervals $I_0$ and $I_1$ such that
\begin{equation*}
    x_{n_0}(I_0) = V_{a^\star}^0 \quad x_{n_0}(I_1) = U_{a^\star}^0 
\end{equation*}
Since $n_0 \geq m_0$ we know that
\begin{equation*}
    \max_{x_0 \in I_0} \abs{x_{n_0+2}(x_0) - f^{2}(x_{n_0}(x_0), a^\star, b)} \leq \epsilon_0
\end{equation*}
Given that $x_{n_0}(I_0) = V_{a^\star}^0$ and $f^2(V_{a^\star}^0, a^\star, b)$ is an $\epsilon_0$ neighborhood of $K_{a^\star} \cup J_{a^\star}$, we can infer that
\begin{equation*}
    x_{n_0+2}(I_0) \supseteq K_{a^\star} \cup J_{a^\star}
\end{equation*}
With a similar analysis we can prove that
\begin{equation*}
    x_{n_0+2}(I_1) \supseteq K_{a^\star} \cup J_{a^\star}
\end{equation*}
Following the same steps repeatedly we can prove for any $k>0$ that
\begin{equation*}
    x_{n_0+2k}(I_0) \supseteq K_{a^\star} \cup J_{a^\star} \quad x_{n_0+2k}(I_1) \supseteq K_{a^\star} \cup J_{a^\star}
\end{equation*}
Since $n_1 = n_0 + 2k$ for some $k>0$ we can directly infer that
\begin{equation*}
    x_{n_1}(I_0) \supseteq K_{a^\star} \cup J_{a^\star} \quad x_{n_1}(I_1) \supseteq K_{a^\star} \cup J_{a^\star}
\end{equation*}
Hence we can construct minimal closed intervals $I_{00},I_{01} \subset I_0$ and $I_{10},I_{11} \subset I_1$ such that
\begin{equation*}
    x_{n_1}(I_{00}) = V_{a^\star}^1 \quad x_{n_1}(I_{01}) = U_{a^\star}^1 \quad x_{n_1}(I_{10}) = V_{a^\star}^1 \quad x_{n_1}(I_{11}) = U_{a^\star}^1
\end{equation*}
By induction, for any binary sequence $c$ of length $i$ we can construct two minimal closed intervals $I_{c_0 \cdots c_{i-1}0}$ and $I_{c_0 \cdots c_{i-1}1}$ subsets of $I_{c_0 \cdots c_{i-1}}$ such that
\begin{equation*}
    x_{n_i}(I_{c_0 \cdots c_{i-1} 0}) = V_{a^\star}^i \quad x_{n_i}(I_{c_0 \cdots c_{i-1} 1}) = U_{a^\star}^i
\end{equation*}
Now for any sequence $A_i$ such that $A_i = V_{a^\star}^i$ or $A_i = U_{a^\star}^i$ we can construct the corresponding sequence $c_i$ that has $c_i=0$ when $A_i= V_{a^\star}^i$ and $c_i=1$ when $A_i= U_{a^\star}^i$. Let us define
\begin{equation*}
    I_{c} = \cap_{i=0}^\infty I_{c_0 \cdots c_i}
\end{equation*}
By Cantor's intersection theorem, $I_{c}$ is non empty. Any $x_0 \in I_{c}$ satisfies the requirements of the lemma
\begin{equation*}
    \forall i \geq 0 \quad x_0 \in I_{c_0 \cdots c_i} \implies \forall i \geq 0 \quad x_{n_i}(x_0) \in A_i
\end{equation*}
\end{proof}

It is important to note that Lemma~\ref{lemma:tracking} ensures the ability to construct the same sequence of $n_i$ for any distinct sequence of $A_i$. Fig. \ref{fig: lem12} shows the construction described in the proof. Formalizing the outlined proof sketch, our final result follows:

\begin{theorem}\label{theorem:dynamic-li-yorke}
If $b \in (0, 1) \setminus \{\frac{1}{2} \}$, there exists a $d_b$ such that if $a_{\min} > d_b$, the dynamics of Equation~\ref{eq:old-update} are Li--Yorke chaotic. 
\end{theorem}
\begin{proof}
We first prove that there is an uncountable set $S$ of infinite length binary sequences with the following property: For every pair of sequences $i,j \in S$ there is an infinite length subsequence where the two sequences differ, i.e., one of $i,j$ is $0$ and the other is $1$.

We first define the equivalence relation $R$ over infinite length binary sequence such that two binary sequences are equivalent if and only if they differ in finitely many places. The relation is clearly reflexive, a sequence differs with itself in $0$ places, it is by definition symmetric and it is transitive, if a sequence $i$ differs in $k$ places with sequence $j$ and $j$ differs in $m$ places with a sequence $t$ then $i$ and $t$ differ in up to $m+k$ places which is also finite. Thus we can partition all binary sequences in equivalence classes where all pairs of all elements in the same class differ in finite places.

We now prove that each equivalence class has countably infinite binary sequences. Let us pick an element $\sigma$ of the equivalence class. For each member of the equivalence class $w$ we can construct a finite subset of $\mathbb{N}$ by picking the indices where $\sigma_t \neq w_t$. Inversely, for each finite subset of $\mathbb{N}$ we can construct a member of the equivalence by flipping the corresponding indices of $\sigma$. Thus there is a bijection between the members of the equivalence class and the finite subsets of $\mathbb{N}$. Because the finite subsets of $\mathbb{N}$ are countably infinite, so is each equivalence class.   

This allows us to prove that the relation $R$ has uncountably infinite number of equivalence classes. We proceed by contradiction. If we had countably many equivalence classes each having countably many elements then all of the binary sequences would be countable which is false.

To construct $S$ we only need to pick one element from each equivalence class. Since each pair belongs in a different equivalent class then they differ in infinite number of places which forms an infinite length subsequence. The set $S$ is uncountable because there are uncountably many equivalence classes.

We now prove that there is an uncountable set $S'$ of infinite length binary sequences that has the following two properties: First, for every pair of sequences $i,j \in S'$ there is an infinite length subsequence where the two sequences differ, i.e., one of $i,j$ is $0$ and the other is $1$. Second, for every pair of sequences $i,j \in S'$ there is an infinite length subsequence where the two sequences are equal, i.e., they are either both $0$ or $1$.

The construction works as follows: For each element of $i \in S$ we construct a new sequence. In its even places the sequence is $0$ and in its odd places we use the elements of $i$. Clearly by construction $S'$ remains uncountable and satisfies the first property as we can pick the subsequences from the odd places. The second property also holds because all sequences have the same elements in the even places.

For each element of $S'$ we construct a sequence of sets $T$ as follows: If the $k$th place element is $0$ we use $V_{a^\star}^k$, whereas if it is $1$ we pick $U_{a^\star}^k$. We now apply Lemma~\ref{lemma:tracking} for each of the sequence of sets to get a corresponding initialization $x_0$ and a subsequence of iteration indices $n_i$. We call this set of initializations $Q$. By the construction in Lemma~\ref{lemma:tracking} all initializations use the same subsequence $n_i$.

Let us pick two initializations $\mu, \nu \in Q$. By the construction of $T$ and $S'$ we have the following: There is an infinite subsequence of $n_i$, which we call $h_t$, where $x_{n_{h_t}}(\mu) \in V_{a^\star}^{h_t}$ and $x_{n_{h_t}}(\nu) \in U_{a^\star}^{h_t}$ or vice versa. Because $V_{a^\star}^{h_t}$ and $U_{a^\star}^{h_t}$ are disjoint we have that there is an infinite subsequence where the trajectories of $\mu$ and $\nu$ are bounded away from each other. In other words, we have that
\begin{equation*}
    \limsup_{n\to \infty} |x_n(\mu) - x_n(\nu)| > 0
\end{equation*}
Again by construction of $T$ and $S'$ we also have that there is an infinite subsequence of $n_i$, which we call $f_t$ such that either $x_{n_{f_t}}(\mu) \in V_{a^\star}^{f_t}$ and $x_{n_{f_t}}(\nu) \in V_{a^\star}^{f_t}$ or $x_{n_{f_t}}(\mu) \in U_{a^\star}^{f_t}$ and $x_{n_{f_t}}(\nu) \in U_{a^\star}^{f_t}$. Thus we have that
\begin{equation*}
    \forall t\geq 0 : |x_{n_{f_t}}(\mu) - x_{n_{f_t}}(\nu)| \leq \max\{\textrm{diam}(U_{a^\star}^{f_t}), \textrm{diam}(V_{a^\star}^{f_t})\}
\end{equation*}
By Lemma~\ref{lemma:turbulence-advanced}, we have that the right hand side converges to $0$. As a result
\begin{equation*}
    \lim_{t \to \infty} |x_{n_{f_t}}(\mu) - x_{n_{f_t}}(\nu)|  = 0
\end{equation*}
which directly implies that
\begin{equation*}
    \liminf_{n \to \infty} |x_n(\mu) - x_n(\nu)|  = 0.
\end{equation*}
Technically our dynamics are expressed in terms of two state variables $x_n$ and $r_n$ so distances in orbits need to account for both dimensions. Because $r_n \to 0$ for all initializations $x_0 \in (0,1)$, in all cases above we have $\lim_{n \to \infty} |r_n(\mu) - r_n(\nu)| = 0$ so the distances in the limit are not affected by the $r_n$ dimension.  

Picking $\{ (x, r) : x \in Q \text{ and } r= 0 \}$ as our scrambled set, which is equivalent to choosing $a_0 = g(0)$ for all $x \in Q$, we show that our dynamic learning rate dynamics are Li--Yorke chaotic. 
\end{proof}

In this work we have focused on the dynamics of \cref{eq:old-update}. But our proof strategy can be readily generalized to any rule $a_n(\cdot)$ as long as it uniformly converges to a sufficiently high constant rate.
\begin{corollary}
Let $a_n(\cdot)$ be a sequence of maps uniformly converging in $[0,1]$ to $a^\star> d_b$. Then the resulting dynamic learning rate system is Li--Yorke chaotic.
\end{corollary}

\section{A note on the attractors of the system and their bifurcations}
\label{sec:sec 6}
\begin{figure}[t!]
    \centering
  \begin{subfigure}[b]{0.43\linewidth}
    \centering
    \includegraphics[width=\linewidth]{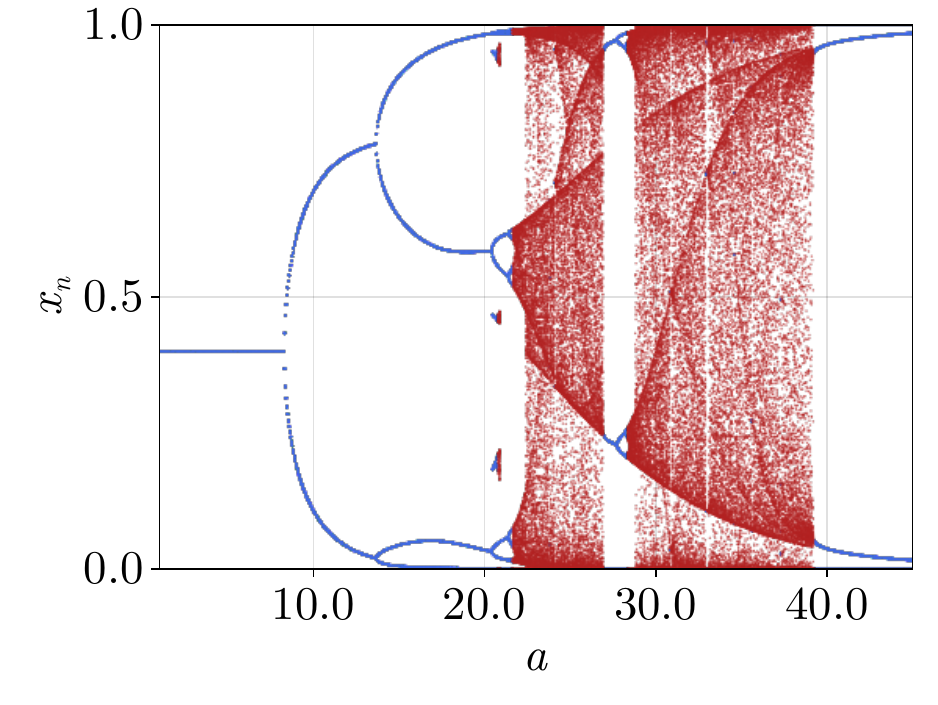}
    \caption{}
  \end{subfigure}
  \hfill
  \begin{subfigure}[b]{0.43\linewidth}
    \centering
    \includegraphics[width=\linewidth]{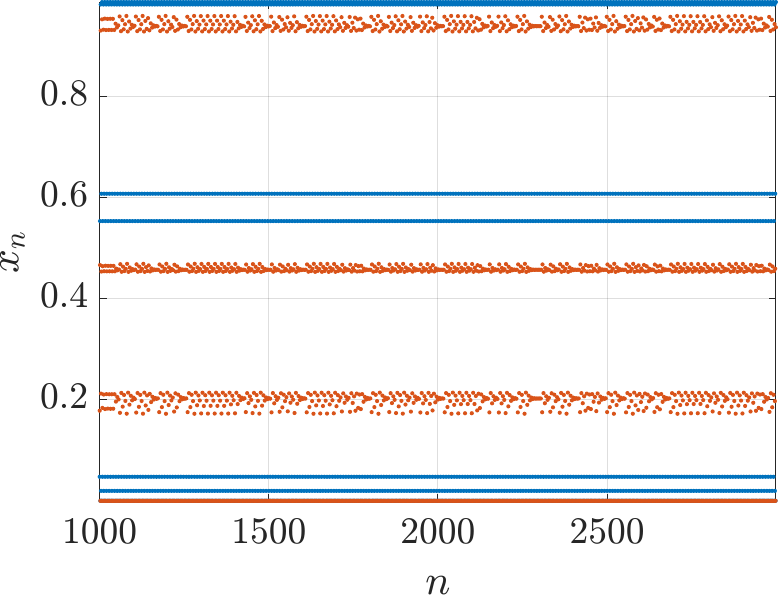}
    \caption{}
  \end{subfigure}
  
  \begin{subfigure}[b]{0.7\linewidth}
    \centering
    \includegraphics[width=\linewidth]{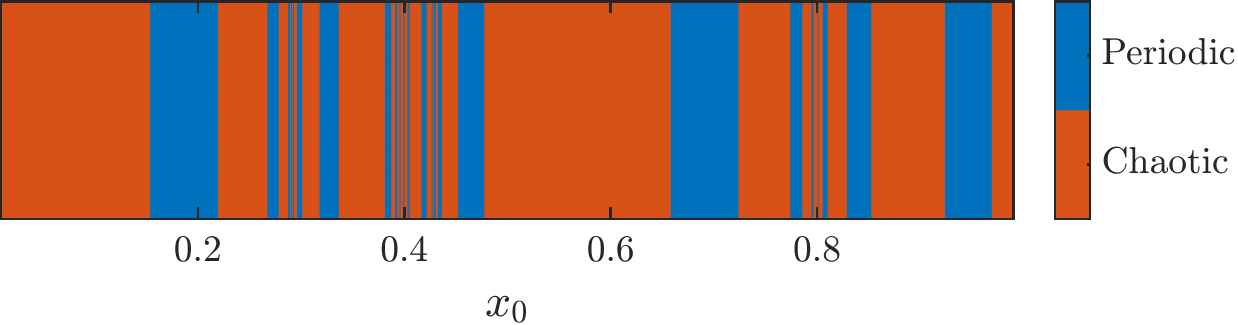}
    \caption{}
  \end{subfigure}
  \begin{subfigure}[b]{0.43\linewidth}
    \centering
    \includegraphics[width=\linewidth]{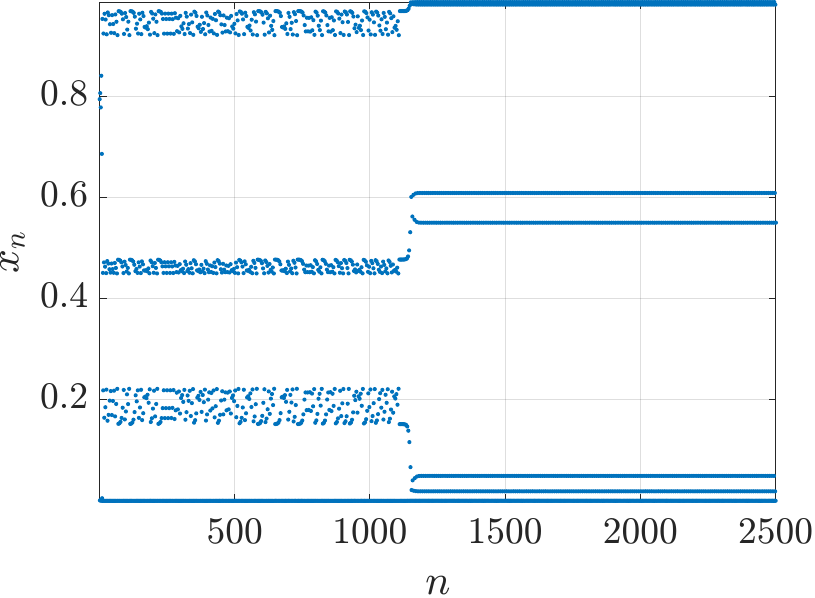}
    \caption{}
  \end{subfigure}
  
  \caption{(a) Bifurcation diagram of \eqref{sys_example} for $b=0.4$, $\beta=0.01$, and $\Delta_a=a_c/4$. The alternation between regular and chaotic behavior is particularly notable, as is the coexistence of periodic and chaotic attractors. Orbits colored in blue correspond to regular behavior, characterized by a Lyapunov exponent associated with the strategy dynamics $\lambda_x<0$. Orbits colored in red correspond to chaotic behavior, with $\lambda_x\geq 10^{-2}$. (b) A periodic attractor coexisting with a chaotic attractor for $a_c=26.85$. (c) The basins of attraction of the two attractors shown in (b). (d) Transient chaos for $a_c=20.93$ and $x_0=0.795.$}
\label{bif1}
%\vspace*{12pt}
\end{figure}

\begin{figure*}[t]
    \centering
    \includegraphics[width=0.8\textwidth]{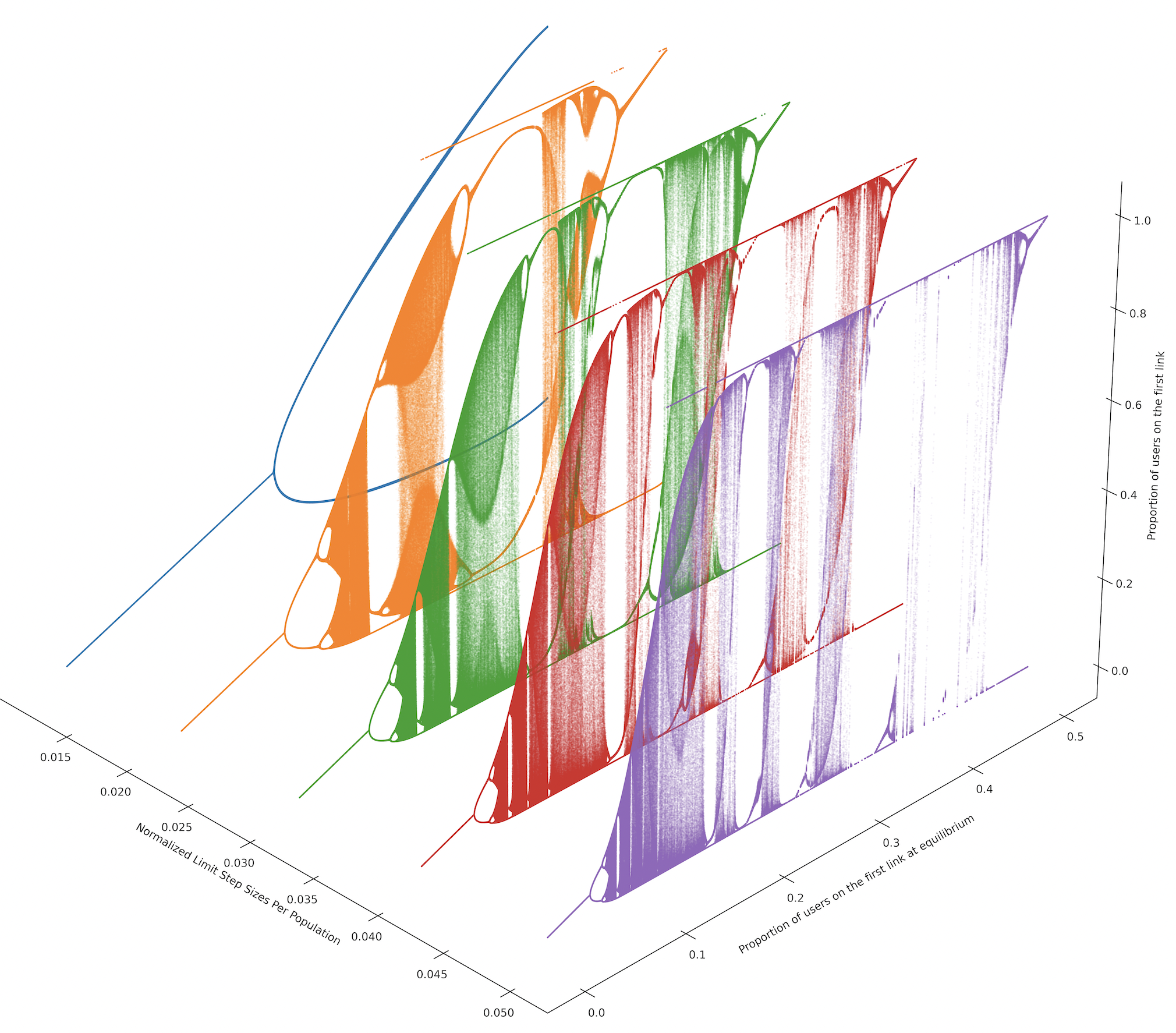}
  \caption{An assortment of bifurcation diagrams for different adaptive learning rate schemes. Larger volatility (larger limit step-size) leads to more prominently unpredictable, chaotic behavior (dense limit-sets for large fraction of the parameter space). 
  While the limit of the scheme becomes more volatile, the behavior of the dynamics tends to be more chaotic as mixed Nash equilibrium moves away from the boundaries.
  }
  \label{fancy}
\end{figure*}

Theorem \ref{theorem:dynamic-li-yorke} establishes the existence of an invariant chaotic set in the sense of Li--Yorke for the adaptive learning map \eqref{eq:old-update}, for any parameter $b\in(0,1)\setminus\{\frac{1}{2}\}$, provided that $a_{\text{min}}$ is sufficiently large. However, we stress that this chaotic set need not be attracting and may even coexist with a regular attractor, such as a periodic orbit. In case it is non-attracting, it frequently appears as a chaotic saddle and transient chaos is typically observed. As a result of the uniform convergence established in Lemma \ref{lemma:dynamic-strong-convergence}, it is natural to expect that for given values of $a^\star$ and $b$, the attractors of the adaptive system \eqref{eq:old-update} coincide with those of the standard MWU system with $a=a^\star$. Nevertheless, the adaptive learning rate may cause an orbit of \eqref{eq:old-update} to converge to a different attractor from that reached by the standard MWU system with $a=a^\star$ and the same initial condition, if multiple attractors coexist. In some cases, these two attractors may even be of completely different dynamical nature, for example, one periodic and the other chaotic. In fact, the potentially intricate structure of the basins of attraction may cause nearby initial conditions to converge to qualitatively different attractors. Consequently, an orbit of the adaptive system \eqref{eq:old-update} cannot, in general, be regarded simply as an orbit of the standard MWU rule with $a=a^\star$ and the same initial condition. 

To highlight the above points and illustrate the different dynamical regimes of \eqref{eq:old-update}, we consider the case where $g(r)=a_c-\Delta_a \tanh{ \left( \beta r \right) }$, which clearly satisfies the properties of Section \ref{sec:prelims} with $a_{\text{min}}=a_c-\Delta_a$ and $a_{\text{max}}=a_c+\Delta_a.$ As a result, the dynamical system reads
\begin{align}
x_{n+1}
&=
\frac{x_n}
{x_n+(1-x_n)\exp\!\left(a_n(x_n-b)\right)},\nonumber \\
a_{n+1}
&=
a_c-\Delta_a\,
\tanh\!\left[
\frac{\beta}{n+1}
\sum_{i=0}^{n}a_i(x_i-b)
\right]\nonumber,
\end{align}
\Needspace{9\baselineskip}
which can be rewritten as an autonomous system in the following form
\begin{align}
x_{n+1}
&=
\frac{x_n}
{x_n+(1-x_n)\exp\!\left(a_n(x_n-b)\right)},\nonumber \\
a_{n+1}
&=
a_c-\Delta_a\,
\tanh\!\left[
(1-s_n)
\operatorname{tanh}^{-1}\!\left(
\frac{a_c-a_n}{\Delta_a}
\right)
+\beta s_n a_n(x_n-b)
\right],\nonumber \\
s_{n+1}
&=
\frac{s_n}{1+s_n}. 
\label{sys_example}
\end{align}

In this case, we have $a_n\to a_c$. Let $b=0.4$, $\beta=0.01$, and $\Delta_a=a_c/4$. Figure \ref{bif1}(a) shows the bifurcation diagram of \eqref{sys_example} as $a_c$ varies. Orbits shown in blue correspond to regular behavior, characterized by a Lyapunov exponent associated with the strategy dynamics $\lambda_x<0$, whereas orbits shown in red correspond to chaotic behavior, with $\lambda_x\geq 10^{-2}$. For small values of $a_c$, the point $b$ is stable and subsequently loses stability through a period-doubling bifurcation. A cascade of period-doubling bifurcations then leads to the emergence of a chaotic attractor. Before this chaotic attractor is created, however, a distinct period-$4$ orbit is born through a saddle-node bifurcation and subsequently undergoes its own period-doubling cascade to chaos. The resulting chaotic attractor coexists with the periodic attractor, as shown in Fig.~\ref{bif1}(b). In fact, for the initial condition $x_0=0.2$, the orbit converges to the chaotic attractor under the adaptive system \eqref{sys_example}, whereas it converges to the periodic attractor under the standard MWU rule with $a=a_c$. Figure \ref{bif1}(c) shows the basins of attraction of the two attractors as a function of the initial condition $x_0$. The chaotic attractor illustrated in Fig.~\ref{bif1}(b) subsequently disappears through a crisis bifurcation, leaving the system with a single periodic attractor. As $a_c$ is increased further, the dynamics undergo additional bifurcations, giving rise to alternating periodic windows and chaotic regimes. Within the periodic windows the chaotic set survives as a chaotic saddle resulting in large chaotic intervals, i.e., the dynamics is transiently chaotic, as shown in Fig. \ref{bif1}(d).

Finally, to illustrate how these phenomena depend not only on the limiting learning rate $a^\star$ but also on the underlying game, i.e. on the parameter $b,$ Fig.~\ref{fancy} presents an assortment of bifurcation diagrams across different values of the equilibrium parameter $b$. Each colored sheet should be interpreted as a standard bifurcation diagram in which the limiting step size $a^\star$ is fixed, and $b$ is varied. As the limiting step size increases, increasingly complicated limit sets occupy a larger portion of the parameter range, with broad regions of apparently chaotic behavior. Moreover, for a fixed limiting step size, the dynamics generally become more intricate as the mixed Nash equilibrium moves away from the boundary of the strategy space. Thus, the bifurcation structure observed in Fig.~\ref{fancy} for the representative choice $b=0.4$ is not an isolated phenomenon, but is part of a richer organization of periodic and chaotic regimes across both the learning-rate and game parameters.

%----------------------------------------------------------------------
%%% DISCUSSION
%----------------------------------------------------------------------
\section{Conclusion}
\label{sec:discussion}
%----------------------------------------------------------------------
%%% INTRODUCTION
%----------------------------------------------------------------------
% !TEX root = ../main_new.tex

%This is the introduction, and this is a test citation of \citet{Nas51}.

We have formally analyzed and established chaotic behavior for a class of multi-agent learning systems with a heuristically updated, variable learning rate. Interestingly, the resulting system is non-autonomous, requiring substantial departures from prior approaches to establishing Li--Yorke chaos in games, such as the celebrated period-3 methodology~\cite{liyorke}, which relies on repeated iterations of a fixed, time-invariant map and is therefore not directly applicable in the present setting.
%To do so, we were inspired by intuition and modelling decisions based on classic AI works in the area such as~\citet{bowling2002multiagent,bowling2004convergence} while simultaneously extending deep technical insights about Li-Yorke chaos in games, which has received a lot of attention recently, popping up in numerous application areas from congestion games \cite{palaiopanos2017multiplicative,chotibut2019route,bielawski2021follow}, to market economies \cite{cheung2021learning}, blockchain mechanism design \cite{leonardos2021dynamical} and even performative prediction~\cite{piliouras2022multi}.
%At the technical crux of all prior formal analysis of Li-Yorke chaos in games (e.g., \cite{bielawski2021follow,chotibut2019route,leonardos2021dynamical,cheung2021learning,piliouras2022multi})  lied the 
%celebrated methodology based on period three orbits~\cite{liyorke}, which is only applicable in autonomous, i.e., time-invariant systems.
In contrast, our analysis requires a deeper examination of the geometry and structural properties of the dynamics, which themselves evolve over time. Nevertheless, we show that a rigorous analysis of chaotic behavior remains possible in this setting. This opens the door to extending existing results to more realistic time-varying models and to establishing rigorous notions of chaos in such systems as well.

%One natural direction for future work is to study to what extent such techniques can be applied in the other application areas encountered above, e.g., to better understand the evolution of markets and blockchain mechanisms under volatility of their exogenous parameters. Another exciting direction is to explore different classes of multi-agent based systems with a variable learning rate, e.g., based on Follow-the-Regularized-leader dynamics~\cite{bielawski2021follow} or Proximal Policy Optimization~\cite{ratcliffe2019win}.  It would also be interesting to identify settings where our motivating question admits a positive answer, i.e., novel classes of dynamically updated learning rates which whilst allowing fast escape from non-promising regions also converge to equilibria in games with large populations of agents.

%It is well known that this technique has shown promise both theoretically and experimentally in small games and particularly in the special case of two agents, two strategy games, however . 

%----------------------------------------------------------------------
%%% ACKNOWLEDGMENTS
%----------------------------------------------------------------------
\section*{Acknowledgments}
\begingroup
%----------------------------------------------------------------------
%%% THANKS
%----------------------------------------------------------------------
% !TEX root = ./main_new.tex
%
%

Emmanouil V. Vlatakis-Gkaragkounis is grateful for financial support by the  Post-Doctoral FODSI-Simons Fellowship, Pancretan Association of America and Simons Collaboration on Algorithms and Geometry and Onassis Doctoral Fellowship. 
%This project was completed while he was a Post-Doctoral fellow at the Simons Institute for the Theory of Computing.
%\ackperiod
This research/project is also supported in part by the National Research Foundation, Singapore and DSO National Laboratories under its AI Singapore Program (AISG Award No: AISG2-RP-2020-016), NRF 2018 Fellowship NRF-NRFF2018-
07, NRF2019-NRF-ANR095 ALIAS grant, grant PIESGP-AI-2020-01, AME Programmatic Fund (Grant No.A20H6b0151) from the Agency for Science, Technology and Research (A*STAR) and Provost’s
Chair Professorship grant RGEPPV2101.\ackperiod

%
%PM is grateful for financial support by the French National Research Agency (ANR) in the framework of
%the ``Investissements d'avenir'' program (ANR-15-IDEX-02),
%the LabEx PERSYVAL (ANR-11-LABX-0025-01),
%MIAI@Grenoble Alpes (ANR-19-P3IA-0003),
%and the grant ALIAS (ANR-19-CE48-0018-01)

\endgroup

%**********************************************************************
%***    BIBLIOGRAPHY
%**********************************************************************
\bibliographystyle{icml}
\bibliography{Bibliography/Bibliography-MV,Bibliography/Bibliography-PM,Bibliography/IEEEabrv,Bibliography/refer,Bibliography/refs}

%**********************************************************************
%***    APPENDICES
%**********************************************************************
\onecolumn
\appendix
\numberwithin{equation}{section}		% for numbering  in the appendix
\numberwithin{lemma}{section}		% for numbering  in the appendix
\numberwithin{fact}{section}		% for numbering  in the appendix
\numberwithin{proposition}{section}		% for numbering  in the appendix
\numberwithin{theorem}{section}		% for numbering in the appendix
\numberwithin{corollary}{section}		% for numbering  in the appendix

%----------------------------------------------------------------------
%%% APP: CONTENTS
%----------------------------------------------------------------------
% \renewcommand{\contentsname}{Organization of the appendix}
% \addtocontents{toc}{\protect\setcounter{tocdepth}{2}}
% \tableofcontents

%----------------------------------------------------------------------
%%% APP: TEST
%----------------------------------------------------------------------
% \section{Test appendix}
% \label{app:test}
% \input{Appendix/App-Test.tex}
\section{Omitted Proofs of Section \ref{sec:fixed}}
\label{app:fixed}
In this section we provide the technical proofs omitted of Section \ref{sec:fixed}. We focus on the case where the dynamical system has a fixed learning rate.
\begin{equation*}
    x_{n+1} = f(x_n, a, b) = \frac{x_n}{x_{n} + (1-x_{n}) \exp(a (x_{n}-b))}
\end{equation*}
The derivative of $f$ for a fixed $a$ and $b$ is
\begin{equation*}
    f'(x, a, b) = \frac{(ax^2 - ax +1) \exp(a(x-b))}{(x+ (1-x) \exp(a(x-b)))^2}
\end{equation*}
Critical/stationary points of $f$ are solutions of $ax^2 - ax +1 = 0$. By taking the discriminant, we get that for $0<a \leq 4$, there is no solution so $f$ is increasing and no chaos can exist, instead the system converges to equilibrium. We will thus require $a>4$ to enable chaotic behaviour in the system. Let us study the local minima and maxima of $f$:
\begin{equation*}
    x_{\max}(a) = \frac{1}{2} - \sqrt{\frac{1}{4} - \frac{1}{a}} \quad x_{\min}(a) = \frac{1}{2} + \sqrt{\frac{1}{4} - \frac{1}{a}}
\end{equation*}
\subsection{The order of \texorpdfstring{$\{ f_{min}(a), x_{\max}(a), b, x_{\min}(a), f_{max}(a)\}$}{\{f\_min(a), x\_max(a), b, x\_min(a), f\_max(a)\}} in the fixed learning rate regime.}
The following facts establish the preliminary necessary observation to prove  structural Lemma~\ref{lemma:structure}. More precisely, we can show the following straightforward facts:
\begin{fact}
For every $b \in (0,1)$, there is a $s_b > 4$ such that for all $a > s_b$
\begin{equation*}
    x_{\max}(a) < b < x_{\min}(a).
\end{equation*}
\end{fact}
\begin{proof}
We have $x_{\max}(a)$ is decreasing and $\lim_{a \to \infty} x_{\max}(a) = 0$. Symmetrically we have $x_{\min}(a)$ is increasing and $\lim_{a \to \infty} x_{\min}(a) = 1$. The fact follows immediately.
\end{proof}
\begin{fact}
If $a>4$, then $f(x_{\min}(a), a, b) = f_{\min}(a) < f(x_{\max}(a), a, b) = f_{\max}(a)$.
\end{fact}
\begin{proof}
Since the function is decreasing after the local maximum and there is no stationary point until the local minimum, the local minimum has smaller value.
\end{proof}
\begin{fact}
If $a > s_b$, then $f_{\min}(a) < b < f_{\max}(a)$.
\end{fact}
\begin{proof}
For $a > s_b$ we know that $f_{\min}(a) < f(b, a, b) = b$ since  $f$ is decreasing in $[b,x_{\min}(a)]$. Symmetrically, $f_{\max}(a) > f(b, a, b) = b$ since again $f$ is decreasing in $[x_{\max}(a), b]$. 
\end{proof}

Having presented the aforementioned intuitive facts, we are ready to prove
Lemma~\ref{lemma:structure}
\begin{lemma}[Restated Lemma~\ref{lemma:structure}]\label{app:lemma:structure}
For every $b \in (0,1)$, there is a $a_b$ such that for all $a> a_b$
\begin{equation*}
    f_{min}(a) < x_{\max}(a)< b < x_{\min}(a) < f_{max}(a)
\end{equation*}
and $f_{\min}(a)$ is decreasing and  $f_{\max}(a)$ is increasing.
\end{lemma}
\begin{proof}
$f_{min}(a) < x_{\max}(a)$ is equivalent to
\begin{equation*}
     x_{\min}(a) - x_{\min}(a) x_{\max}(a) - x_{\max}^2(a)\exp\left[a (x_{\min}(a) -b)\right] < 0 
\end{equation*}
Observe that the first two terms are bounded and that for the third term we have
\begin{equation*}
    \lim_{a \to \infty} x_{\max}^2(a)\exp\left[a (x_{\min}(a) -b)\right] = \infty
\end{equation*}
So we can pick an $a_b$ large enough so that $a> a_b$ implies that the inequality holds. The case of $x_{\min}(a) < f_{max}(a)$ is symmetric. Moving on to the monotonicity of $f_{\min}(a)$, to avoid overloading notation let us call $x_1$ the $x$ argument of $f$ and $x_2$ its $a$ argument
\begin{equation*}
    \frac{\partial f_{\min}(a)}{\partial a} = \frac{\partial f ( x_{\min}(a), a, b)}{\partial x_1} \frac{\partial x_{\min} (a)}{\partial a} + \frac{\partial f ( x_{\min}(a), a, b)}{\partial x_2}
\end{equation*}
Observe that since $x_{\min}(a)$ is a local minimum, we have that 
\begin{equation*}
    \frac{\partial f ( x_{\min}(a), a, b)}{\partial x_1} = 0
\end{equation*}
Additionally we can pick $a_b$ large enough so that $a> a_b$ implies $x_{\min}(a) > b$. Let us write $f ( x_{\min}(a), a, b)$ 
\begin{equation*}
    f ( x_{\min}(a), a, b) = \frac{x_{\min}(a)}{x_{\min}(a) + (1-x_{\min}(a)) \exp(a(x_{\min}(a)-b))}.
\end{equation*}
Treating $x_{\min}(a)$ as a constant independent of $a$, the exponential in the denominator is an increasing function of $a$ so $f ( x_{\min}(a), a, b)$ is decreasing with respect to $a$. This makes 
\begin{equation*}
    \forall a > a_b \quad \frac{\partial f ( x_{\min}(a), a, b)}{\partial x_2} < 0
\end{equation*}
As a result we have that 
\begin{equation}
 \forall a > a_b  \quad \frac{\partial f_{\min}(a)}{\partial a} < 0   
\end{equation}
and  $f_{\min}(a)$ is decreasing. The case of $f_{\max}(a)$ is symmetric. We take the maximum of all the required $a_b$ to get the result.
\end{proof}

\subsection{Forward invariant, Perpetual \& Absorbing sets for fixed learning rate.}
Having settled the order among $ \{ f_{min}(a), x_{\max}(a), b, x_{\min}(a), f_{max}(a)\} $ for high enough learning rates, we are ready to prove the  $(i)$ forward invariant, $(ii)$ perpetual and $(iii)$ absorbing property of $\mathcal{F}(a)=[f_{\min}(a), f_{\max}(a)]$.

For the sake of readability, we recall first the formal definitions of these properties:
\begin{enumerate}
\item 
\textbf{A forward invariant set}: a set of states such that if the system starts in any state in the set, it will remain in the set for all future time.
\item 
\textbf{A perpetual set}: a special case of a forward invariant set, whose image is itself.
\item 
\textbf{An (global/local) absorbing set}: a forward invariant set that also includes (globally/locally) all possible future states of the system.
\end{enumerate}

\begin{lemma}[Restated Lemma~\ref{lemma:forward}]\label{app:lemma:forward}
For every $b \in (0,1)$ there is a $s_b$ such that  $\mathcal{F}(a)$ is forward invariant for all $a> s_b$, i.e. $$
    x \in \mathcal{F}(a) \Rightarrow f(x, a, b) \in \mathcal{F}(a).$$
\end{lemma}

\begin{proof}
By continuity of $f$, the image of $[f_{\min}(a), f_{\max}(a)]$ is determined by four points: $f_{\min}(a)$, $f_{\max}(a)$, $x_{\max}(a)$, and $x_{\min}(a)$. Since $f(x, a, b) \geq x$ in $(0, b)$ and $f_{\max}(a)$ is the maximum in this interval, we know that
\begin{equation*}
    f_{\max}(a) \geq f(f_{\min}(a), a, b) \geq f_{\min}(a)
\end{equation*}
Since $f(x, a, b) \leq x$ in $(b, 1)$ and $f_{\min}(a)$ is the minimum in this interval, we know that
\begin{equation*}
    f_{\max}(a) \geq f(f_{\max}(a), a, b) \geq f_{\min}(a)
\end{equation*}
Of course the images of the local optima $x_{\max}(a)$, $x_{\min}(a)$ trivially belong to $[f_{\min}(a), f_{\max}(a)]$.
\end{proof}
But even points that are outside of this interval are monotonically attracted to it without overshooting. We have the following lemma that shows this monotonic attracting property
\begin{lemma}
Let $a> s_b$, then
\begin{align*}
    x \in (0, f_{\min}(a)) &\implies f(x, a, b) \in (x, f_{\max}(a)] \\
    x \in (f_{\max}(a), 1) &\implies f(x, a, b) \in [f_{\min}(a), x).
\end{align*}
\end{lemma}
\begin{proof}
We will prove the first one, the second one is entirely symmetric. If $x < f_{\min}(a)$ then $x<b$ and thus $f(x, a, b) > x$ and $f(x,a, b) \leq f_{\max}(a)$. The first implication follows immediately. 
\end{proof}

Consequently, we have the following lemma:
\begin{lemma}[Restated Lemma~\ref{lemma:perpetual}]\label{app:lemma:perpetual}
For $a > a_b$, $\mathcal{F}(a)$ is surjective, i.e. 
$
    f(\mathcal{F}(a), a, b) = \mathcal{F}(a).
$
\end{lemma}
\begin{proof}
We already know that the right hand side is a super set of the left hand side by \Cref{app:lemma:forward}. But by \Cref{app:lemma:structure} we have that $x_{\max}(a)$ and $x_{\min}(a)$ belong to $[f_{\min}(a), f_{\max}(a)]$. So the left-hand side is a superset of the right-hand side. Thus the result follows.
\end{proof}

More generally we can prove the following absorbing condition:
\begin{lemma}[Restated Lemma~\ref{lemma:absorbption}]\label{app:lemma:absorbption}
Let $a> s_b$,  $
      x_0 \in [\gamma, \zeta] \Rightarrow \exists n_0 : \ \forall n \geq n_0 \text{ s.t } \ x_{n} \in \mathcal{F}(a)
$,  for every $[\gamma, \zeta] \subset (0,1)$.
\end{lemma}
\begin{proof}
If $x_0 \in [f_{\min}(a), f_{\max}(a)]$ then the results follows trivially for $n_0 = 0$. Let us take the case
\begin{equation*}
    x_0 \in \left[\gamma, f_{\min}(a)\right) = I_0
\end{equation*}
We know that in $I_0$ we have that $f(x, a, b) > x$. We can define a uniform bound on their difference
\begin{equation*}
    \min_{x \in I_0} \left[ f(x, a, b) -x \right] = \xi > 0
\end{equation*}
We now know that
\begin{equation*}
    x_1 \in \left[\gamma + \xi, f_{\max}(a)\right)
\end{equation*}
If $x_1 \in \mathcal{F}(a)$ then the result follows trivially for $n_0 = 1$. Otherwise we have that
\begin{equation*}
    x_1 \in \left[\gamma + \xi, f_{\min}(a)\right) \subset I_0
\end{equation*}
Applying recursively, either there is a $n_0$ such that $x_{n_0} \in \mathcal{F}(a)$ and the theorem follows trivially or for all $n$ we have that $x_n$ stays in a subset of $I_0$
\begin{equation*}
    x_n \in \left[\gamma + n\xi, f_{\min}(a)\right) \subset I_0
\end{equation*}
But this is impossible since there is a $n_0>0$ such that 
\begin{equation*}
 \forall n \geq n_0 : \gamma + n\xi >   f_{\min}(a) 
\end{equation*}
We are now left with the symmetric case of 
\begin{equation*}
    x_0 \in \left(f_{\max}(a), \zeta\right]
\end{equation*}
which we can handle just like above.
\end{proof}

\subsection{Two period trajectories in MWU maps.}
We start with a fundamental observation from Calculus of continuous injective functions
\begin{claim}\label{app:claim:period-two}
Let $g$ be a continuous decreasing function on a closed interval $I$ and that there exists a $x_0 \in I$ such that $\forall n: g^n(x_0) \in I$. Then 
%  \begin{align*}
%     g(\limsup_{n \to \infty} g^n(x_0)) = \liminf_{n \to \infty} g^n(x_0)\\
%     g(\liminf_{n \to \infty} g^n(x_0)) = \limsup_{n \to \infty} g^n(x_0)
% \end{align*}
either $g^n(x_0)$ converges to a fixed point $x^\star$ or to a 2-period trajectory.
\end{claim}
\begin{proof}
Let us take $g^{n_j}(x_0)$ a subsequence that converges to $\liminf_{n \to \infty} g^n(x_0)$.
\begin{equation*}
    \limsup_{n \to \infty} g^n(x_0) \geq \lim_{j \to \infty} g^{n_j+1}(x_0)  =  g(\liminf_{n \to \infty} g^{n}(x_0))  
\end{equation*}
Now let us take $g^{n_k}(x_0)$ a subsequence that converges to  $\limsup_{n \to \infty} g^n(x_0)$. Since $g$ is decreasing and thus invertible in $I$ we know that $g^{n_k -1}(x_0)$ also converges
\begin{equation*}
    \limsup_{n \to \infty} g^n(x_0) = g(\lim_{k \to \infty} g^{n_k-1}(x_0)) \leq   g(\liminf_{n \to \infty} g^{n}(x_0))
\end{equation*}
The two steps clearly imply that
\begin{equation*}
    \limsup_{n \to \infty} g^n(x_0) = g(\liminf_{n \to \infty} g^{n}(x_0))
\end{equation*}
Symmetrically, with the same arguments we have that
\begin{equation*}
    \liminf_{n \to \infty} g^n(x_0) = g(\limsup_{n \to \infty} g^{n}(x_0))
\end{equation*}
Thus, if we denote
\begin{align*}
    x_1 &= g(\limsup_{n \to \infty} g^n(x_0)) = \liminf_{n \to \infty} g^n(x_0), \\
    x_2 &= g(\liminf_{n \to \infty} g^n(x_0)) = \limsup_{n \to \infty} g^n(x_0),
\end{align*}
then either $x_1=x_2=x^\star$, which is a fixed point for map $g$ \textendash $g(x^\star)=x^\star$, or $x_1\neq x_2$ consist a 2-period trajectory $\{x_1=g(x_2),x_2=g(x_1)\}$.
\end{proof}

Interestingly, if we restrict our attention to the interval $\mathcal{D}=[x_{\max}(a), x_{\min}(a)]$, we can show that there is no 2-period trajectory:
\begin{lemma}\label{app:lemma:period-two-trajectory}
For every $b \in (0,1)$, there is a $\ell_b$ such that for all $a > \ell_b$, there is no period two trajectory for which both endpoints belong to $[x_{\max}(a), x_{\min}(a)]$.
\end{lemma}

\begin{proof}
Endpoints of period two trajectories satisfy the equation $f^2(x, a, b) = x$
\begin{equation*}
    f^2(x, a, b) = \frac{x}{x + (1-x) \exp(a ( x + f(x, a, b) -2b))} = x
\end{equation*}
Ignoring $x=0$ and $x=1$, which are merely fixed points, this is equivalent to
\begin{equation*}
    x + f(x, a, b) = 2b
\end{equation*}
After some manipulation, the formula above is equivalent to
\begin{equation*}
    \gamma_{a,b}(x) = (2b-x-1)x + (2b-x)(1-x)\exp(a(x-b)) = 0
\end{equation*}
We take the first and second and third derivative of this function
\begin{align*}
    \gamma_{a,b}'(x) = 2b -2x -1 + \exp(a(x-b)) \left[ (2x-2b-1) + a(2b-x)(1-x))\right] \\
    \gamma_{a,b}''(x) = -2 + \exp(a(x-b)) \left[ a^2(2b-x)(1-x) + 2a(2x-2b-1) + 2\right]\\
    \gamma_{a,b}'''(x) = a \exp(a(x-b)) \left[ a^2(2b-x)(1-x) + 3a(2x-2b -1) + 6\right]
\end{align*}
Let us define the following finite quantity
\begin{equation*}
\mu_{b} = \frac{\max_{x \in [0, b]}3(2b+1-2x)}{\min_{x \in [0, b]}(2b-x)(1-x)}    
\end{equation*}
For $a > \mu_b$, we have that $\gamma_{a,b}'''(x) > 0$ in $[0, b]$. Moving on to $\gamma_{a,b}''(x)$, it is increasing in $[0,b]$ and
\begin{equation*}
    \gamma_{a,b}''(0) = -2 + \exp(-ab)\left[2b a^2 -2a(2b-1) + 2\right] \quad \gamma_{a,b}''(b) = a^2b(1-b) -2a
\end{equation*}
Clearly we have
\begin{equation*}
    \lim_{a \to \infty} \gamma_{a,b}''(0) = - 2 \quad \lim_{a \to \infty} \gamma_{a,b}''(b) =  \infty
\end{equation*}
We can thus pick an $\kappa_b > \mu_b$ such that for all $a > \kappa_b$
\begin{equation*}
    \gamma_{a,b}''(0)  < 0 \quad \gamma_{a,b}''(b) > 0.
\end{equation*}
Thus $\gamma_{a,b}''(x)$ has exactly one root in $[0, b]$ given its monotonicity. Moving on to $\gamma_{a,b}'(x)$, it starts off as decreasing and moves to increasing in $[0,b]$ with
\begin{align*}
    \gamma_{a,b}'(0) &= 2b -1 + \exp(-ab) \left[ (-2b-1) + 2ab\right]\\
    \gamma_{a,b}'(b) &= -2 + ab(1-b)
\end{align*}
To continue our analysis we will study the following cases: $b< \frac{1}{2}$, $b = \frac{1}{2}$, $b > \frac{1}{2}$. 

\paragraph{Case: $b< \frac{1}{2}$} For the first case
\begin{equation*}
    \lim_{a \to \infty} \gamma_{a,b}'(0) = 2b -1 < 0 \quad  \lim_{a \to \infty} \gamma_{a,b}'(b) = \infty
\end{equation*}
Thus we can pick a $\nu_{b} > \kappa_b$ such that for $a> \nu_{b}$, $\gamma_{a,b}'(0) < 0$ and $\gamma_{a,b}'(b)> 0$. Since $\gamma_{a,b}'(x)$ starts decreasing and moves to increasing, it has exactly one root. Moving on to $\gamma_{a,b}(x)$ we have
\begin{equation*}
    \gamma_{a,b}(0) = 2b \exp(-ab) > 0 \quad \gamma_{a,b}(b) = 0
\end{equation*}
Given that $\gamma_{a,b}(x)$ starts decreasing and moves to increasing in $[0,b]$, it can have up to two roots in $[0,b]$, one of which is $b$.  We can observe that
\begin{equation*}
    \lim_{a\to \infty} \gamma_{a,b}\left(\frac{b}{2}\right) =  \left(\frac{3b}{2} - 1\right) \left(\frac{b}{2}\right)< 0.
\end{equation*}
We can pick a $\xi_{b} > \nu_{b}$ such that for all $a> \xi_b$ it holds that $\gamma_{a,b}\left(\frac{b}{2}\right) < 0$. For $a> \xi_b$ we have that $\gamma_{a,b}(x)$ has exactly one root in $[0, b)$ that is located in $[0,\frac{b}{2}]$.

\paragraph{Case: $b = \frac{1}{2}$} For the second case we have that
\begin{equation*}
    \gamma_{a,b}'(0) = \exp\left(-\frac{a}{2}\right)(a- 2) \quad \lim_{a\to \infty} \gamma_{a,b}'\left(\frac{b}{2}\right) = b -1 < 0   \quad \lim_{a \to \infty} \gamma_{a,b}'(b) = \infty
\end{equation*}
So we can pick a $\nu_{b} > \kappa_b$ such that for $a> \nu_{b}$, $\gamma_{a,b}'(0) > 0$ and $\gamma_{a,b}'\left(\frac{b}{2}\right) < 0$ and $\gamma_{a,b}'(b)> 0$. Since $\gamma_{a,b}'(x)$ starts decreasing and moves to increasing, it has exactly two roots. Moving on to $\gamma_{a,b}(x)$ we have
\begin{equation*}
    \gamma_{a,b}(0) = 2b \exp(-ab) > 0 \quad \gamma_{a,b}(b) = 0
\end{equation*}
In $[0, b]$, we have that $\gamma_{a,b}(x)$ starts increasing and positive, then switches to decreasing and then to increasing. In the first section it cannot have any root. In the second section it can have at most one root and in the third section it has exactly one root $b$. Just like above we can observe that
\begin{equation*}
    \lim_{a\to \infty} \gamma_{a,b}\left(\frac{b}{2}\right) =  \left(\frac{3}{4} - 1\right) \left(\frac{1}{4}\right)< 0.
\end{equation*}
Following the same steps as above, we can pick a $\xi_b$ such that for $a> \xi_b$ we have that $\gamma_{a,b}(x)$ has exactly one root in $[0, b)$ that is located in $[0,\frac{b}{2}]$.

\paragraph{Case: $b > \frac{1}{2}$} For the last case we have that
\begin{equation*}
    \lim_{a \to \infty} \gamma_{a,b}'(0) = 2b -1 > 0  \quad \lim_{a\to \infty} \gamma_{a,b}'\left(\frac{b}{2}\right) = b -1 < 0  \quad \lim_{a \to \infty} \gamma_{a,b}'(b) = \infty
\end{equation*}
Just like before we can pick a $\nu_{b} > \kappa_b$ such that for $a> \nu_{b}$, $\gamma_{a,b}'(0) > 0$ and  $\gamma_{a,b}'\left(\frac{b}{2}\right) < 0$ and $\gamma_{a,b}'(b)> 0$. Since $\gamma_{a,b}'(x)$ starts decreasing and moves to increasing, it has exactly two roots. Moving on to $\gamma_{a,b}(x)$ 
\begin{equation*}
    \gamma_{a,b}(0) = 2b \exp(-ab) > 0 \quad \gamma_{a,b}(b) = 0
\end{equation*}
In $[0, b]$, we have that $\gamma_{a,b}(x)$ starts increasing and positive, then switches to decreasing and then to increasing. In the first section it cannot have any root. In the second section it can have at most one root and in the third section it has exactly one root $b$. We can observe that
\begin{equation*}
    \gamma_{a,b}\left(\frac{3b-1}{2}\right) =  \left(\frac{b-1}{2}\right) \left(\frac{3b-1}{2}\right)< 0.
\end{equation*}
Following the same steps as above, we can pick a $\xi_b$ such that for $a> \xi_b$ we have that $\gamma_{a,b}(x)$ has exactly one root in $[0, b)$ that is located in $[0,\frac{3b-1}{2}]$.

\paragraph{Solution pairs} 
In all cases, we have identified an $\xi_{b}$ such that for $a> \xi_b$ we have that  $\gamma_{a,b}(x)$ has exactly one root in $[0,b)$. If $\gamma_{a,b}(x) = 0$, then we know that  $\gamma_{a,b}(f(x, a, b)) = 0$. Since $x + f(x, a, b) = 2b$, we know that $f(x, a, b)$ needs to be in $(b, 2b]$. Symmetrically, any root of  $\gamma_{a,b}(x) = 0$ with $x> b$ can only form a periodic trajectory with an $f(x, a, b) = 2b-x < b$ that is also satisfies $\gamma_{a,b}(f(x, a, b)) = 0$. But since there is only one root in $[0,b)$ and points cannot participate in multiple periodic trajectories, we have a unique solution for $\gamma_{a,b}(x) = 0$ in $(b, 1)$. Let us define $x_l(a) < b < x_r(a)$ the points of the unique two-periodic trajectory as functions of $a$. These functions are bounded and thus they need to have at least one limit point. We can use the following equations
\begin{equation*}
    \lim_{a \to \infty} \gamma_{a,b}(x_l(a)) = \lim_{a \to \infty} \gamma_{a,b}(x_r(a)) = 0 
\end{equation*}
to derive the properties of these limit points. In all three cases above, we proved that the solution $x_l(a)$ is bounded away from $b$ for $a> \xi_b$. As a result all limit points must satisfy 
\begin{equation*}
    (2b-x-1)x = 0 \implies x = 0 \textrm{ or } x = 2b-1.
\end{equation*}
Similarly, since $x_l(a) + x_r(a) = 2b$  and $x_l(a)$ is bounded away from $b$ , it must be the case that $x_r(a)$ is bounded away from $b$. As such all limit points of $x_r(a)$ must satisfy
\begin{equation*}
    (2b-x)(1-x) = 0 \implies x = 2b \textrm{ or } x = 1.
\end{equation*}
Observe that the limit points of $x_l(a)$ and $x_r(a)$ must come in pairs that sum to $2b$ as $x_l(a) + x_r(a) = 2b$. We will once again do a case by case study. For $b < \frac{1}{2}$, the only viable pair is $(0,2b)$ since $2b-1 < 0$. For $b = \frac{1}{2}$, there is only one pair $(0,1)$. For the case of $b > \frac{1}{2}$, the only viable pair is $(2b-1, 1)$ since $2b > 1$. In all of the cases, the limit points of $x_r(a)$ and $x_l(a)$ are unique and thus $x_r(a)$ and $x_l(a)$ converge.

\paragraph{Convergence rate} We are now ready to argue why at least one of $x_r(a)$ and $x_l(a)$ do not belong in $[x_{\max}(a), x_{\min}(a)]$ for sufficiently large $a$. Once again we will do a case by case analysis on $b$. For $b < \frac{1}{2}$, we will argue that there is a $\ell_b$ such that for $a > \ell_b$ we have that $x_l(a) < x_{\max}(a)$. We know that $x_l(a) = f(x_r(a), a, b)$. Thus $x_l(a) < x_{\max}(a)$ is equivalent to
\begin{equation*}
    x_r(a) - x_{\max}(a) x_r(a) - x_{\max}(a)(1- x_r(a))\exp\left[a (x_r(a) -b)\right] < 0
\end{equation*}
Observe that the first two terms are bounded but for the third term we have
\begin{equation*}
    \lim_{a \to \infty} x_{\max}(a)(1- x_r(a))\exp\left[a (x_r(a) -b)\right] = \infty
\end{equation*}
given that $\lim_{a \to \infty} x_r(a) = 2b$ and the exponential goes to $\infty$ much faster than $x_{\max}(a)$ goes to 0. Thus we can choose a $\ell_b$ such that for $a > \ell_b$ we have $x_l(a) < x_{\max}(a)$. For the case of $b>\frac{1}{2}$, we use the same arguments to prove that there is an $\ell_b$ such that for $a> \ell_b$ we have $x_r(a) > x_{\min}(a)$. For the case of $b=\frac{1}{2}$, we will argue that there is a $\ell_b$ such that for $a > \ell_b$ we have that $x_l(a) < x_{\max}(a)$. We need to study the convergence rate of $x_r(a)$ to $1$. We have that 
\begin{equation*}
    \gamma_{a, \frac{1}{2}}(x_r(a)) = 0 \implies x_r(a) = (1-x_r(a)) \exp\left[\frac{a}{2} \left(x_r(a) -\frac{1}{2}\right)\right]
\end{equation*}
We can apply the same argument as in the $b < \frac{1}{2}$ case and use the equation above to prove that
\begin{equation*}
    \lim_{a \to \infty} x_{\max}(a)(1- x_r(a))\exp\left[a \left(x_r(a) -\frac{1}{2}\right)\right] = \lim_{a \to \infty} x_{\max}(a)x_r(a)\exp\left[\frac{a}{2} \left(x_r(a) -\frac{1}{2}\right)\right] = \infty
\end{equation*}
because $\lim_{a \to \infty} x_r(a) = 1$ and the exponential goes to $\infty$ much faster than $x_{\max}(a)$ goes to 0. The resulting $\ell_b$ in all cases satisfy the requirements of the theorem. 
\end{proof}

\subsection{Volume Expansion \& Instability of mixed equilibrium \texorpdfstring{$x^\star=b$}{x*=b}}
\begin{lemma}\label{app:lemma:instability_of_b}
For every $b \in (0,1)$, there is a $k_b$ such that for $a> k_b$ it holds that
there exists some neighborhood $ \mathcal{N}_{\delta}=(b-\delta,b+\delta)$, such that almost all initializations from $\mathcal{N}_{\delta}$ do not converge to $x^\star=b$.
\end{lemma}
\begin{proof}
We can write down $f'(b, a, b)$ as
\begin{equation*}
    f'(b, a, b) = ab^2 -ab +1
\end{equation*}
Since $\lim_{a \to \infty}f'(b, a, b) = -\infty$, thus 
there is a $k_b$ such that for $a> k_b$ it holds that
there exists some neighborhood $ \mathcal{N}_{\delta}=(b-\delta,b+\delta)$ where $|f'(x, a, b)|>1 \ \forall \ x\in\mathcal{N}_{\delta}$. Leveraging Unstable Manifold Theorem (See \cite{Shub}), the result immediately follows.
\end{proof}

We are now ready to prove that any neighborhood of $x^\star=b$ eventually covers the perpetual set.
\begin{lemma}\label{lemma:basis-expansion}
For every $b \in (0,1)$, there is a $z_b$ such that for all $a > z_b$.  
\begin{equation*}
    \{ b\} \subset [\gamma, \zeta] \subseteq [x_{max}(a),x_{\min}(a)] \implies \exists n_0 : \quad \forall n \geq n_0 \quad f^n([\gamma, \zeta], a, b) = [f_{\min}(a), f_{\max}(a)]
\end{equation*}
\end{lemma}
\begin{proof}
Let us take the following interval
\begin{align*}
    R &= f([\gamma,b], a, b) \cap [b, \zeta] \\
      &= [f(b, a, b), f(\gamma, a, b)] \cap [b, \zeta] \\
      &= [b, f(\gamma, a, b)] \cap [b,\zeta] \\
      &= [b, \min \{f(\gamma, a, b) , \zeta\}]
\end{align*}
Let us pick an $x_0 \in R$. Let us assume that
\begin{equation*}
    x_{\max}(a) \leq f^n(x_0, a, b) \leq x_{\min}(a)
\end{equation*}
By Lemma~\ref{app:lemma:instability_of_b} and picking $a> k_b$ we can always choose a $x_0 \in R$ such that $f^n(x_0,a,b)$ does not converge to $b$. Since $b$ is the unique fixed point of $f(x,a,b)$ in $[x_{\max}(a), x_{\min}(a)]$, we have that $f^n(x_0, a, b)$ does not converge at all and as a result
\begin{equation*}
    \liminf_{n \to \infty} f^n(x_0, a, b) \neq \limsup_{n \to \infty} f^n(x_0, a, b)
\end{equation*}
By Claim~\ref{app:claim:period-two} on $f$, which is decreasing on $[x_{\max}(a), x_{\min}(a)]$,  we have that the points above form a two-period trajectory of $f(x, a, b)$. By Lemma~\ref{app:lemma:period-two-trajectory}, if we pick $\ell_b$ such that for $a > \ell_b$ no period two trajectory can exist inside $[x_{\max}(a), x_{\min}(a)]$ yielding a contradiction. As a result
\begin{equation*}
    \exists n^* : \quad f^{n^*}(x_0, a, b) \notin [x_{\max}(a), x_{\min}(a)]
\end{equation*}
We will first study the case of $f^{n^*}(x_0, a, b) > x_{\min}(a)$. For this case, we have
\begin{equation*}
    f^{n^*}([\gamma, \zeta], a, b) \supseteq f^{n^*}([b, x_0], a, b) \supseteq [b, x_{\min}(a)].
\end{equation*}
In the next iteration
\begin{equation*}
    f^{n^*+1}([\gamma, \zeta], a, b) \supseteq [f_{\min}(a), b].
\end{equation*}
Observe though that because $x_0 \in R$, we can pick a $x_1 \in [\gamma, b]$ such that $f(x_1, a, b) = x_0$. We then have
\begin{equation*}
    f^{n^*+1}(x_1, a, b) = f^{n^*}(x_0, a, b) > x_{\min}(a)
\end{equation*}
This yields
\begin{equation*}
    f^{n^*+1}([\gamma, \zeta], a, b) \supseteq [f_{\min}(a), x_{\min}(a)]
\end{equation*}
For $a > a_b$ by Lemma~\ref{lemma:structure} $f_{\min}(a) < x_{\max}(a)$ and by Lemma~\ref{lemma:forward} we have that $[f_{\min}(a), f_{\max}(a)]$ is forward invariant. As a result we have the following
\begin{equation*}
    [f_{\min}(a), f_{\max}(a)] \supseteq f^{n^*+2}([\gamma, \zeta], a, b) \supseteq [f_{\min}(a), f_{\max}(a)]
\end{equation*}
By Lemma~\ref{app:lemma:perpetual} we know that $[f_{\min}(a), f_{\max}(a)]$ is perpetual. So we can choose $n_0 = n^* + 2$ to fulfill the requirements of the theorem. The case of $f^{n^*}(x_0, a, b) < x_{\max}(a)$ is symmetric.
\end{proof}

We are ready now to prove the main volume expansion claims of the section.
\begin{theorem}[Restated Theorem~\ref{thm:volume-expansion}]\label{app:thm:volume-expansion}
For every $b \in (0,1)$, there is a $v_b$ such that for all $a > v_b$ and 
 any interval $[\gamma, \zeta]$  : $  \{ b\} \subset [\gamma, \zeta] \subset (0,1) $, it holds that
\begin{equation*}
 \exists n_0 : \ \forall n \geq n_0 \quad f^n([\gamma, \zeta], a, b) = \mathcal{F}(a).
\end{equation*}
\end{theorem}
\begin{proof}
We can apply Lemma~\ref{lemma:basis-expansion} on $s = [\gamma, \zeta] \cap [x_{\max}(a), x_{\min}(a)] \supset \{b\}$. Then there is a $n^*_1$ such that
\begin{equation*}
    \forall n \geq n^*_1 \quad f^n([\gamma, \zeta], a, b) \supseteq f^n(s, a, b) = [f_{\min}(a), f_{\max}(a)]
\end{equation*}
But by Lemma~\ref{lemma:absorbption}, for $a > s_b$ there is a $n^*_2$ such that
\begin{equation*}
    \forall n \geq n^*_2 \quad [f_{\min}(a), f_{\max}(a)] \supseteq f^n([\gamma, \zeta], a, b) 
\end{equation*}
Taking $v_b = \max \{ z_b, s_b\}$ and $n_0 = \max \{n^*_1, n^*_2 \}$ satisfies the theorem requirements.
\end{proof}

\clearpage
\section{Omitted Proofs of Section \ref{sec:dynamic}}
\label{app:dynamic}
\subsection{Uniform convergence of varying step-size MWU map and volume expansion.}

\begin{lemma}[Restated Lemma~\ref{lemma:dynamic-strong-convergence}]\label{app:lemma:dynamic-strong-convergence}
Let $a_{\min}> s_b$. For every $k$, $\epsilon>0$ and $[\gamma, \zeta] \subset (0,1)$
\begin{equation*}
    \exists n_0: \forall n \geq n_0 \quad \max_{x_0 \in [\gamma, \zeta]} \abs{x_{n+k}(x_0) - f^k(x_n(x_0), a^\star, b)} \leq \epsilon.
\end{equation*}
\end{lemma}
\begin{proof}
%In a completely similar fashion we can prove the result for any iterate $k$.
For the sake of readability, we will present the case $k=2$ and with a
completely similar fashion we can prove the result for any iterate $k$.
    \begin{lemma}
Let $a_{\min}> s_b$. For every $\epsilon>0$ and $[\gamma, \zeta] \subset (0,1)$
\begin{equation*}
    \exists n_0: \forall n \geq n_0 \quad \max_{x_0 \in [\gamma, \zeta]} \abs{x_{n+2}(x_0) - f^2(x_n(x_0), a^\star, b)} \leq \epsilon.
\end{equation*}
\end{lemma}
\begin{proof}
By Corollary~\ref{corollary:bound} there is a $\delta > 0$ such that for any $x_0 \in [\gamma, \zeta]$, we have that $x_n(x_0) \in (\delta, 1-\delta)$. Let us define the following function that takes an $x$ and applies $f$ with learning rates $a_1$ and $a_2$
\begin{equation*}
    q(x, a_1, a_2, b) = f(f(x, a_1, b), a_2, b)
\end{equation*}
With this definition in mind we have that
\begin{equation*}
    D(x_0) = x_{n+2}(x_0) - f^2(x_n(x_0), a^\star, b) = q(x_n(x_0), a_{n}(x_0), a_{n+1}(x_0), b) - q(x_n(x_0), a^\star, a^\star, b)
\end{equation*}
By the mean value theorem we have that there exists a $c \in (0, 1)$ such that 
\begin{align*}
    (a_1, a_2) &= c(a_n(x_0), a_{n+1}(x_0)) + (1-c) (a^\star, a^\star)
\end{align*}
and also
\begin{align*}
    D(x_0)= \frac{\partial q(x_n(x_0), a_1, a_2, b)}{\partial a_1}(a_n(x_0) - a^\star) + \frac{\partial q(x_n(x_0), a_1, a_2, b)}{\partial a_2}(a_{n+1}(x_0) - a^\star)
\end{align*}
The following value is finite 
\begin{align*}
    \phi = \max_{\substack{x \in (\delta, 1-\delta), \\\ a_1, a_2 \in (a_{\min}, a_{\max})}} \sqrt{\left(\frac{\partial q(x, a_1, a_2, b)}{\partial a_1}\right)^2 + \left(\frac{\partial q(x, a_1, a_2, b)}{\partial a_2} \right)^2}.
\end{align*}
Then we have
\begin{equation*}
   \abs{x_{n+2}(x_0) - f^2(x_n(x_0), a^\star, b)} \leq \phi \sqrt{(a_n(x_0) - a^\star)^2 + (a_{n+1}(x_0) - a^\star)^2}
\end{equation*}
By the uniform convergence of $a_n(x_0)$ and $a_{n+1}(x_0)$ to $a^\star$, the result follows immediately.
\end{proof}
\end{proof}

\begin{lemma}[Restated Lemma~\ref{lemma:dynamic-learning:volume-expansion-basic}]\label{app:lemma:dynamic-learning:volume-expansion-basic}
For every $b \in (0,1)$, there is a $z_b$ such that for all $a_{\min} > z_b$ it holds that for every $[\gamma, \zeta]$  such that $\{ b\} \subset [\gamma, \zeta] \subseteq \mathcal{D}(a_{\min})=[x_{\max}(a_{\min}), x_{\min}(a_{\min})]$
\begin{equation*}
     \exists n_0: \ \forall n \geq n_0 \ \  x_{n}([\gamma, \zeta]) \supseteq %[f_{\min}(a_{\min}), f_{\max}(a_{\min})]=
     \mathcal{F}(a_{\min}).
\end{equation*}
\end{lemma}
\begin{proof}
Our first observation is that for two $a_1 >  a_2 > 0$
\begin{equation}\label{eq:simple:increase-alpha}
\begin{aligned}
    1> x > b &\implies f(x, a_1, b) < f(x, a_2, b) \\
    0< x < b &\implies f(x, a_1, b) > f(x, a_2, b) \\
    x = b &\implies f(x, a_1, b) = f(x, a_2, b) = b
\end{aligned}
\end{equation}
Let us define the following subset of $[\gamma, \zeta]$
\begin{equation*}
    [\gamma', \zeta'] = [\gamma, \min \{f(\gamma, a_{\min}, b) , \zeta\}]
\end{equation*}
Using the same arguments as in \Cref{lemma:basis-expansion} we know that there is a $x' \in (b, \zeta']$ and a minimal $n^*$ such that
\begin{equation*}
    f^{n^*}(x', a_{\min}, b) \not\in [x_{\max}(a_{\min}), x_{\min}(a_{\min})]
\end{equation*}
We are going to assume $f^{n^*}(x', a_{\min}, b) > x_{\min}(a_{\min})$. The case of $f^{n^*}(x', a_{\min}, b) < x_{\max}(a_{\min})$ is entirely symmetric. Since $n^*$ is minimal we have that
\begin{equation*}
    \forall 0 \leq n < n^* \quad f^n(x', a_{\min}, b) \in [x_{\max}(a_{\min}), x_{\min}(a_{\min})]
\end{equation*}
As a result we have that
\begin{align*}
    \forall 0 \leq 2n < n^* \quad f^{2n}([b, x'], a_{\min}, b) =&  [b, f^{2n}(x', a_{\min}, b)] \\
    \forall 0 < 2n +1 < n^* \quad f^{2n +1}([b, x'], a_{\min}, b) =&  [f^{2n+1}(x', a_{\min}, b), b]
\end{align*}
Using the above equations as well as \Cref{eq:simple:increase-alpha} we can recursively prove that
\begin{equation*}
    \forall 0 \leq n < n^* \quad x_n([b, x']) \supseteq  f^n([b, x'], a_{\min}, b)
\end{equation*}
We can deduce that
\begin{equation*}
    x_{n^*}([b, x']) \supseteq [b, f^{n^*}(x', a_{\min}, b)] \supseteq [b, x_{\min}(a_{\min})]
\end{equation*}
With one more iteration we have
\begin{equation*}
    x_{n^*+1}([b, x']) \supseteq [f_{\min}(a_{\min}), b]
\end{equation*}
We can also pick a $x'' \in [\gamma', b]$ such that $f(x'', a_{\min}, b) = x'$. By construction, we know that $n^*+1$ is the first iteration such that
\begin{equation*}
    f^{n^*+1}(x'', a_{\min}, b) \not\in [x_{\max}(a_{\min}), x_{\min}(a_{\min})].
\end{equation*}
By similar arguments as above we can prove that
\begin{equation*}
    x_{n^*+1}([x'', b]) \supseteq [b, f^{n^*+1}(x'', a_{\min}, b)] = [b, f^{n^*}(x', a_{\min}, b)] \supseteq [b, x_{\min}(a_{\min})]
\end{equation*}
As a result we now have
\begin{equation*}
    x_{n^*+1}([x'', x']) \supseteq  [f_{\min}(a_{\min}), x_{\min}(a_{\min})] \supseteq [x_{\max}(a_{\min}), x_{\min}(a_{\min})]
\end{equation*}
It follows directly that
\begin{equation*}
    x_{n^*+2}([\gamma, \zeta]) \supseteq x_{n^*+2}([x'', x']) \supseteq [f_{\min}(a_{\min}), f_{\max}(a_{\min})] \supseteq [x_{\max}(a_{\min}), x_{\min}(a_{\min})]
\end{equation*}
Applying the above step recursively we get
\begin{equation*}
    \forall n \geq n^*+2 \quad x_{n}([\gamma, \zeta]) \supseteq [f_{\min}(a_{\min}), f_{\max}(a_{\min})]
\end{equation*}
so $n_0=n^*+2$ satisfies the requirements of the theorem.
\end{proof}

\begin{theorem}[Restated Theorem~\ref{theorem:dynamic-volume-expansion}]\label{app:theorem:dynamic-volume-expansion}
For $a_{\min} > z_b$ and for any sufficiently small $\epsilon > 0$ such that  $a^\star-\epsilon >a_{\min} > z_b$ we have that for all $[\gamma, \zeta]$ such that 
$\{ b\} \subset [\gamma, \zeta] \subset (0,1)$, it holds that
\begin{equation*}
      \exists n_0: \ \forall n \geq n_0 \quad x_{n}([\gamma, \zeta]) \supseteq \mathcal{F}(a^\star-\epsilon).
\end{equation*}
\end{theorem}
\begin{proof}
We will imitate the proof strategy of the above lemma by showing that $ a_{\min}$ could be substitute by any $\tilde{a}=a^\star-\epsilon$
for any sufficiently small $\epsilon > 0$ such that  $a^\star-\epsilon >a_{\min} > z_b$. Our first observation is that thanks to the uniform convergence of $a_n \to a^\star$ in $[\gamma, \zeta] \subset (0,1)$ there exists a $n^{\dag}$ such that 
\[
\forall n \geq n^{\dag} \quad a_{n}([\gamma, \zeta]) \subseteq (a^\star-\epsilon, a^\star+\epsilon)
\]
which implies that  there exists a $n^{\dag}$ such that
\begin{equation}\label{eq:lower-bound-alpha}
    \forall n > n^{\dag} \quad \forall x \in [\gamma, \zeta] \quad a_n(x) > a^\star-\epsilon = \tilde{a}
\end{equation}

Again notice that  for any $a_n(x) >  \tilde{a} > 0$
\begin{equation}\label{eq:increase-alpha}
\begin{aligned}
    1> x > b &\implies f(x, a_n(x), b) < f(x,  \tilde{a}, b) \\
    0< x < b &\implies f(x, a_n(x), b) > f(x,  \tilde{a}, b) \\
    x = b &\implies f(x, a_n(x), b) = f(x,  \tilde{a}, b) = b
\end{aligned}
\end{equation}
Let us define the following intervals
\begin{align*}
    [r, q] &= x_{n^{\dag}}([\gamma, \zeta]) \cap [x_{\max}(\tilde{a}), x_{\min}(\tilde{a})] \supset \{b\}\\
    [\gamma', \zeta'] &= [r, \min \{f(r, \tilde{a}, b) , q\}]
\end{align*}
Using the same arguments as in \Cref{lemma:basis-expansion} we know that there is a $x' \in (b, \zeta']$ and a minimal $n^*$ such that
\begin{equation*}
    f^{n^*}(x', \tilde{a}, b) \not\in [x_{\max}(\tilde{a}), x_{\min}(\tilde{a})]
\end{equation*}
We are going to assume $f^{n^*}(x', \tilde{a}, b) > x_{\min}(\tilde{a})$. The case of $f^{n^*}(x', \tilde{a}, b) < x_{\max}(\tilde{a})$ is entirely symmetric. Since $n^*$ is minimal we have that
\begin{equation*}
    \forall 0 \leq n < n^* \quad f^n(x', \tilde{a}, b) \in [x_{\max}(\tilde{a}), x_{\min}(\tilde{a})]
\end{equation*}
As a result we have that
\begin{align*}
    \forall 0 \leq 2n < n^* \quad f^{2n}([b, x'], \tilde{a}, b) =&  [b, f^{2n}(x', \tilde{a}, b)] \\
    \forall 0 < 2n +1 < n^* \quad f^{2n +1}([b, x'], \tilde{a}, b) =&  [f^{2n+1}(x', \tilde{a}, b), b]
\end{align*}
Using the above equations as well as \Cref{eq:lower-bound-alpha,eq:increase-alpha} we can recursively prove that
\begin{equation*}
    \forall 0 \leq n < n^* \quad x_{n+n^{\dag}}([b, x']) \supseteq  f^n([b, x'], \tilde{a}, b)
\end{equation*}
We can deduce that
\begin{equation*}
    x_{n^{\dag}+n^*}([b, x']) \supseteq [b, f^{n^*}(x', \tilde{a}, b)] \supseteq [b, x_{\min}(\tilde{a})]
\end{equation*}
With one more iteration we have
\begin{equation*}
    x_{n^{\dag}+n^*+1}([b, x']) \supseteq [f_{\min}(\tilde{a}), b]
\end{equation*}
We can also pick a $x'' \in [\gamma', b]$ such that $f(x'', \tilde{a}, b) = x'$. By construction, we know that $n^*+1$ is the first iteration such that
\begin{equation*}
    f^{n^*+1}(x'', \tilde{a}, b) \not\in [x_{\max}(\tilde{a}), x_{\min}(\tilde{a})].
\end{equation*}
By similar arguments as above we can prove that
\begin{equation*}
    x_{n^{\dag}+n^*+1}([x'', b]) \supseteq [b, f^{n^*+1}(x'', \tilde{a}, b)] = [b, f^{n^*}(x', \tilde{a}, b)] \supseteq [b, x_{\min}(\tilde{a})]
\end{equation*}
As a result we now have
\begin{equation*}
    x_{n^{\dag}+n^*+1}([x'', x']) \supseteq  [f_{\min}(\tilde{a}), x_{\min}(\tilde{a})] \supseteq [x_{\max}(\tilde{a}), x_{\min}(\tilde{a})]
\end{equation*}
It follows directly that
\begin{equation*}
    x_{n^{\dag}+n^*+2}([\gamma, \zeta]) \supseteq x_{n^{\dag}+n^*+2}([x'', x']) \supseteq [f_{\min}(\tilde{a}), f_{\max}(\tilde{a})] \supseteq [x_{\max}(\tilde{a}), x_{\min}(\tilde{a})]
\end{equation*}
Applying the above step recursively we get
\begin{equation*}
    \forall n \geq n^{\dag}+n^*+2 \quad x_{n}([\gamma, \zeta]) \supseteq [f_{\min}(\tilde{a}), f_{\max}(\tilde{a})]
\end{equation*}
so $n_0=n^{\dag}+n^*+2$ satisfies the requirements of the theorem.
\end{proof}

\clearpage
\section{Omitted Proofs of Section \ref{sec:turbulence}}
\label{app:turbulence}
In this appendix we prove the existence of a period-3 orbit lying in the interior of $\mathcal{F}(a)$.

\begin{theorem}\label{theorem_turb_1}
 If $b \in (0, 1) \setminus \{\frac{1}{2}\}$, then there exists 
 $u_b$ such that for all $a > u_b$ it holds that the map $f(x,a,b)$ has a period-3 orbit in the interior of $\mathcal{F}(a)=[f_{\min}(a), f_{\max}(a)]$.
\end{theorem}
\begin{proof} 
$f(x,a,b) > x$ is equivalent to $x < b$ and $f^3(x,a,b) < x$ is equivalent to $x + f(x,a,b) + f^2(x,a,b) > 3b$. Assume that $0 < b < 1/2$. Then $3b - 1 < b$, so we can take $x > 0$ such that $3b - 1 < x < b$. Then $f(x,a,b) > x$. Moreover, $\exp(a(x - b))$ goes to $0$ as $a$ goes to infinity, so
\[
\lim_{a\to+\infty} f(x,a,b) = 
\lim_{a\to+\infty} \frac{x}{ x + (1 - x) \exp(a(x - b))} = 1
\]
Thus, since $3b - x < 1$, there exists $u_b$ such that for all $a > u_b$ then $f(x,a,b) > 3b - x$, so $x + f(x,a,b)  +f^2(x,a,b) > 3b$. Hence, if $a > u_b$ then $f^3(x,a,b) < x < f(x,a,b)$. Now,  from the main theorem in \cite{liyorke} it follows that if $a > u_b$ then $f$ has a period-3 point.  Observe that the chosen $x$ depends only on $b$ and not on $a$. Since
\begin{equation*}
    \lim_{a \to \infty} f_{\min}(a) = 0 \quad \lim_{a \to \infty} f_{\max}(a) = 1  
\end{equation*}
we can pick $u_b$ large enough such that for $a > u_b$ $x \in [f_{\min}(a), f_{\max}(a)]$. Picking $u_b > a_b$ we also have that $f(x, a, b)$ and $f^2(x, a, b)$ belong to $[f_{\min}(a), f_{\max}(a)]$ as well for $a > u_b$. The period-3 orbit constructed by \cite{liyorke} has thus all its points in the interior of $[f_{\min}(a), f_{\max}(a)]$. The case of  $1/2 < b < 1$ is symmetric because $f(x, a, b) = f(1-x, a, 1-b)$. 
\end{proof}

\end{document}